%% file: arxiv_version_iclr.tex
\documentclass{article} 
\usepackage{iclr2027_conference,times}

\input{math_commands.tex}

\usepackage[utf8]{inputenc} 
\usepackage[T1]{fontenc}    
\usepackage{hyperref}       
\usepackage{url}            
\usepackage{booktabs}       
\usepackage{amsfonts}       
\usepackage{nicefrac}       
\usepackage{microtype}      
\usepackage{xcolor}         
\usepackage{amsmath, amssymb, amsthm}
\usepackage{mathtools}
\usepackage{bm}
\usepackage{graphicx}
\usepackage{hyperref}
\usepackage{subcaption}
\usepackage{enumitem}

\usepackage{cleveref}

\usepackage{amsmath,amssymb}
\usepackage{algorithm}
\usepackage{algpseudocode}

\usepackage{siunitx}
 \usepackage{multirow}

\theoremstyle{plain}
\newtheorem{theorem}{Theorem}
\newtheorem{lemma}{Lemma}

\newtheorem{corollary}{Corollary}
\theoremstyle{definition}
\newtheorem{definition}{Definition}

\theoremstyle{remark}

\newcommand{\method}{$\gamma$-BIFR }

\title{Beyond Square Roots: A Memory-Efficient Explicit Factorization for Multi-Epoch Private Learning}

\author{%
  Nikita P. Kalinin \\
     Institute of Science and Technology Austria \\
   \texttt{nikita.kalinin@ist.ac.at}
   \And
   Aki Rehn\\
   University of Helsinki \\
   \texttt{aki.rehn@helsinki.fi} \\
   \And
   Joel Daniel Andersson \\
   Institute of Science and Technology Austria \\
   \texttt{joel.andersson@ist.ac.at}    
   \And
   Antti Honkela\\
   University of Helsinki \\
   \texttt{antti.honkela@helsinki.fi} \\
   \And
   Christoph H. Lampert \\
   Institute of Science and Technology Austria \\
   \texttt{chl@ist.ac.at}
}

\iclrfinalcopy 
\begin{document}

\maketitle

\begin{abstract}
Correlated-noise mechanisms are among the most promising approaches for improving the utility of differentially private model training, but rigorous guarantees require explicit, analyzable factorizations, and practical deployment requires memory efficiency. Recent works have developed banded inverse factorizations, which address both requirements by exploiting a banded structure in the correlation matrix. The bandwidth controls the size of the noise buffer used to correlate noise across iterations, and thus governs the tradeoff between utility and memory cost. At the two ends of this tradeoff, DP-$\lambda$CGD achieves high memory efficiency by using only a one-step noise buffer, but this limits its utility gains, while the banded inverse square root (BISR) factorization exploits larger correlation windows and is asymptotically optimal for large bandwidths but performs poorly at low bandwidths. To bridge this gap, we introduce $\gamma$-BIFR, a unified generalization of both factorizations. In the low-memory, low-bandwidth regime, $\gamma$-BIFR improves RMSE, amplified RMSE, and private training performance, while yielding tighter theoretical guarantees for multi-participation error in multi-epoch training.
\end{abstract}

\section{Introduction}
As readily available public data become increasingly exploited, private and sensitive data, such as medical records, financial information, and personal communications, represent an important source of information for machine learning, making strong privacy safeguards essential for their responsible use.
Differential Privacy \citep{dwork2006calibrating} has become the gold standard as a formal notion of privacy. By varying the privacy parameters, one can choose a trade-off between privacy and utility that is appropriate for the problem at hand. In machine learning, differential privacy is typically achieved through variants of differentially private stochastic gradient descent (DP-SGD) \citep{rajkumar2012differentially, abadi2016deep}, which clip per-sample gradients to bound their influence on the model and add \emph{independent} Gaussian noise to ensure privacy.

One promising way to improve the utility of DP-SGD is to introduce correlation in the noise across iterations. By partially canceling noise added in previous steps, such methods can improve the utility of the resulting model. This linear correlation structure is typically described by a \emph{correlation matrix}, and the resulting approach is known as the matrix factorization mechanism \citep{li2015matrix}. Originally developed for linear queries \citep{li2015matrix} and continual counting \citep{dwork2010differential, fichtenberger2023constant, andersson2023smooth, andersson2024count}, this method has since been extended to machine learning \citep{denisov2022improved}, including settings with multi-epoch training \citep{choquette2022multi, choquette2023amplified, kalinin2024}.

To introduce correlations in the noise across iterations, one must have access to previously generated noise, for example by storing it. This creates an additional memory and computational burden for correlated-noise methods. Recent work addresses this issue by imposing structural restrictions on the inverse of the correlation matrix, referred to as the strategy matrix, thereby reducing the memory requirement to storing only a small buffer of noise vectors \citep{dvijotham2024efficient, scalingmckenna2024, andersson2024streaming, henzinger2025binned}. One recent approach imposes a banded structure on the correlation matrix directly, so that only a number of noise vectors equal to the bandwidth minus one must be stored. The resulting factorizations are referred to as banded inverse methods and include the optimization-based BandInvMF \citep{kalinin2025back}, the explicit Banded Inverse Square Root (BISR) \citep{kalinin2025back}, and its application to training with learning-rate scheduling \citep{kalinin2025learningrate}. The main advantage of the banded inverse approach is that it cancels only recently added noise, precisely where the model is most likely to benefit from such cancellation. However, BISR performs well only for relatively large bandwidths. BandInvMF addresses this limitation, but it does not provide theoretical guarantees and performs poorly under amplification by subsampling.

A recent work by \citet{kalinin2026dplambdacgd} introduced DP-$\lambda$CGD, a memory-free and time-efficient factorization that correlates the noise only with one previous iteration. By exploiting the pseudorandom nature of computer-generated noise, the method avoids storing past noise vectors and instead regenerates them when needed. This idea can be applied efficiently to any banded inverse factorization, since such methods rely only on fresh noise. Its main limitation, however, is utility: the proposed method performs worse than alternative memory-efficient approaches that introduce correlations across a larger number of iterations, such as Buffered Linear Toeplitz (BLT) \citep{hasslefree2024}. 

In this paper, we strictly generalize DP-$\lambda$CGD to larger bandwidths by introducing generalized Banded Inverse Fractional Root factorization ($\gamma$-BIFR), which takes an arbitrary matrix power $\gamma$ and recovers BISR when $\gamma = 1/2$. Intuitively, $\gamma$ controls how aggressively previously added noise is canceled: when $\gamma = 0$, the method corresponds to DP-SGD with no noise cancellation, while when $\gamma = 1$, it fully cancels the noise from the previous iteration. Empirically, we observe that the optimal choice of $\gamma$ depends strongly on the available bandwidth. In the small-bandwidth regime, the best-performing value of $\gamma$ is far from $1/2$ and typically much closer to $1$, achieving substantially lower error than BISR under the same bandwidth constraint.

\textbf{Contributions.}

\begin{itemize}[nosep]
    \item We propose a new explicit $\gamma$-BIFR factorization for the low-bandwidth regime, generalizing DP-$\lambda$CGD to larger bandwidths while preserving memory-free noise regeneration and time efficiency; see Section~\ref{sec:gammabifr}.

    \item We prove Root Mean Squared Error (RMSE) guarantees in the multi-participation (multi-epoch) setting; see Theorem~\ref{thm:bifr-multi-participation}.

    \item We analyze amplified RMSE  and show that tuning $\gamma$ achieves lower error than existing methods; see \Cref{sec:amplified_rmse,sub:amplification_by_subsampling}.
    \item We empirically validate the method on CIFAR-10, CIFAR-100 and IMDb, demonstrating accuracy improvements over existing factorization approaches; see Section~\ref{sec:experiments}.
\end{itemize}

\section{Background}

In this paper, we adopt differential privacy as our formal privacy definition. More specifically, we report our privacy guarantees in the form of approximate differential privacy\footnote{The matrix mechanism, as a specific instance of the Gaussian mechanism, naturally satisfies a stronger privacy guarantee, namely Gaussian differential privacy \citep{dong2022gaussian}.}.

\begin{definition}[Approximate Differential Privacy \citep{Dwork_DP_original}]
A randomized algorithm $M$ satisfies $(\varepsilon,\delta)$-differential privacy if, for every pair of datasets $D$ and $D'$ differing in at most one element, and for every measurable subset $S$ of the output space of $M$, the following holds:
\begin{equation}
\Pr[M(D)\in S] \leq e^{\varepsilon}\Pr[M(D')\in S] + \delta.
\end{equation}
\end{definition}

The standard approach to ensuring differential privacy during the training of a machine learning model is Differentially Private Stochastic Gradient Descent (DP-SGD) \citep{abadi2016deep}. Formally, at each step, we sample a batch $B_i$, compute the per-example gradients $g_{i,j}$, and clip them using a predefined clipping norm $\zeta > 0$ to ensure that their $\ell_2$ norm is bounded. We then average the clipped gradients within the batch to obtain an aggregated gradient $\hat{g}_i$. To ensure privacy, we add appropriately scaled Gaussian noise $z_i \in \mathbb{R}^{d}$ with variance $\sigma_{\varepsilon,\delta}^2$, guaranteeing $(\varepsilon,\delta)$-differential privacy. The computation of the noise level $\sigma_{\varepsilon,\delta}$ can be based on the Gaussian mechanism \citep{dwork2006calibrating, balle2018improving}, or can be made more sophisticated by accounting for the randomness introduced by batch sampling (subsampling) \citep{abadi2016deep, koskela2020computing}.

The performance of DP-SGD can be substantially improved, both in the centralized setting (with subsampling) and in the federated learning setting (without subsampling), by using the matrix factorization mechanism \citep{li2015matrix, denisov2022improved}. Rather than adding independent Gaussian noise $z_i$ at each iteration, this method introduces noise that is linearly correlated across iterations. The correlation coefficients are collected in a lower triangular matrix, commonly denoted by $C^{-1}$ and referred to as the ``correlation matrix''. Its inverse, $C$, is referred to as the ``strategy'' matrix and is used to determine the noise level required to maintain differential privacy. Given the correlation matrix $C^{-1}$ and an appropriately chosen noise level $\sigma_{\varepsilon,\delta}$, one can train a differentially private model in much the same way as with DP-SGD, but instead by adding correlated noise of the form $\sum_{j=1}^{i} (C^{-1})_{i,j} z_j$. The intuition behind this method is that, when the goal is to continually estimate partial sums of vectors rather than the vectors themselves, one can add more noise to the individual vectors while subtracting parts of the previously added noise in such a way that the total variance is reduced, and can in fact be asymptotically smaller \citep{dwork2010differential}.

The design of the strategy matrix $C$ turns out to be subtle, with multiple works proposing different structures, or ``factorizations'' (see, e.g., the survey by \citealp{pillutla2025correlated}). The most common design is based on the Root Mean Squared Error (RMSE) \citep{denisov2022improved, choquette2022multi, kalinin2024, scalingmckenna2024}. Specifically, let us concatenate all aggregated clipped gradients into a matrix $\hat{G}$ and all noise vectors into a matrix $Z$. Privacy is then ensured by adding correlated noise of the form $\hat{G} + C^{-1}Z$. However, to compute the intermediate models, the gradients must be summed over time. For this purpose, consider the lower triangular matrix $E \in \mathbb{R}^{n \times n}$, referred to as the prefix-sum matrix:
\begin{equation}
    E = \begin{pmatrix}
        1 & 0 & \cdots & 0 \\
        1 & 1 & \cdots & 0 \\
        \vdots & \vdots & \ddots & \vdots \\
        1 & 1 & \cdots & 1
    \end{pmatrix}.
\end{equation}
The partial sums of the gradients are then given by $E(\hat{G} + C^{-1}Z) = E\hat{G} + EC^{-1}Z,$
which is equivalent to adding correlated noise to the intermediate models with correlation matrix $EC^{-1}$. The RMSE objective then suggests minimizing
\begin{equation}
\label{eq:rmse_def_full}
    \sqrt{\mathbb{E}_Z \|EC^{-1}Z\|_F^2 / n} = \sqrt{d} \|EC^{-1}\|_F \cdot \sigma_{\varepsilon,\delta}(C) \zeta / \sqrt{n},
\end{equation}
which corresponds to the average variance introduced into the partial sums of the gradients. In the non-amplified setting, the noise level $\sigma_{\varepsilon,\delta}(C)$ is given by the product of the Gaussian mechanism noise multiplier for sensitivity $1$, denoted by $\sigma_{\varepsilon,\delta}$ (can be computed tightly numerically \citep{balle2018improving}), and the global sensitivity of the product $C\hat{G}$, namely
\begin{equation}
    \mathrm{sens}(C) = \sup_{G \sim G'} \|CG - CG'\|_F,
\end{equation}
where the streams $G$ and $G'$ differ in the participation of a single data point or user. To compute this sensitivity, it is sufficient to consider non-adaptive streams, meaning that one only needs to account for the direct participation of data points in the gradient computation and may ignore their indirect influence through the model weights in future iterations. For a formal proof of this statement, we refer to Theorem~2.1 in \citet{denisov2022improved}. To compute the sensitivity, one must first specify the participation pattern, namely the iterations in which a data point is allowed to participate. In the non-amplified setting, we adopt the notion of so-called \emph{$b$-min separation} \citep{choquette2023amplified}, whereby the time gap between consecutive participations is at least $b$. One may further restrict the maximum number of participations to be at most $k$, although it is typically assumed that $k$ is as large as the separation parameter allows. For lower triangular Toeplitz matrices with positive, decreasing coefficients along the sub-diagonals, the sensitivity takes a particularly simple form, given by the following theorem.


\begin{theorem}[Sensitivity for lower triangular Toeplitz matrices under $b$-min-separation; Theorem~2 from \citet{kalinin2024}]
\label{thm:b-sensitivity}
Let $C$ be a lower triangular Toeplitz matrix
with decreasing non-negative entries
$c_0 \ge c_1 \ge c_2 \ge \dots \ge c_{n-1} \ge 0$ on the diagonals.
Then the \emph{sensitivity} of $C$ in the setting of \emph{$b$-min-separation} is
\begin{align}
\mathrm{sens}_{k,b}(C)
  &= \left\|\sum_{j=0}^{k-1} C_{[\cdot,\, jb]}\right\|_2 = \sqrt{\sum_{i = 0}^{n - 1}\left[\sum_{j = 0}^{\min(k - 1, \lfloor i/b\rfloor)}c_{i - jb}\right]^2}.
\label{eq:sens_toeplitz}
\end{align}
\end{theorem}

Since the factor $\sqrt{d}\zeta$ in \Cref{eq:rmse_def_full} is independent of the correlation matrix $C^{-1}$, we omit it and, assuming $E=BC$, define the RMSE, with or without amplification, as

\begin{equation}
\label{eq:amplified_rmse}
    \mathrm{RMSE}_{\varepsilon,\delta}(B,C)
    := \|B\|_F \cdot \sigma_{\varepsilon,\delta}(C)/\sqrt{n}
    \overset{\text{(no amp.)}}{=}
    \|B\|_F \cdot \mathrm{sens}(C)\sigma_{\varepsilon,\delta}/\sqrt{n},
\end{equation}

where the second equality holds in the absence of amplification. For our theoretical analysis in the non-amplified setting, we set $\sigma_{\varepsilon,\delta}=1$ to simplify the expressions and remain consistent with previous work, and define
\begin{equation}
\label{eq:rmse_def}
    \mathcal{E}(B,C)
    := \|B\|_F \cdot \mathrm{sens}(C)/\sqrt{n}.
\end{equation}

Given an explicit way of computing the sensitivity and the RMSE objective \Cref{eq:rmse_def}, one can obtain a matrix $C$ via gradient-based optimization \citep{choquette2023amplified}. There are, however, notable limitations to this approach.  First, the factorization obtained through gradient-based methods is essentially a black box and provides limited theoretical insight and guarantees. Second, in order to correlate the noise across iterations, one must store the noise in a buffer, which requires a significant memory overhead compared with DP-SGD. Several works have addressed these limitations in the multi-participation setting, including Banded Inverse Square Root (BISR) \citep{kalinin2025back} and DP-$\lambda$CGD \citep{kalinin2026dplambdacgd}.

Both methods belong to the class of \emph{banded inverse} factorizations, in which the correlation matrix $C^{-1}$ has bandwidth $p$: only the main diagonal and the first $p-1$ subdiagonals can be nonzero. Consequently, the noise added at a given iteration depends only on noise generated during the most recent $p-1$ iterations, reducing the memory requirement. BISR constructs these coefficients from the inverse square root of the prefix-sum matrix $E^{-1/2}$. DP-$\lambda$CGD instead uses the smallest nontrivial bandwidth $p=2$, with $C^{-1}$ having $1$ on the main diagonal and $-\lambda$ on the first subdiagonal, so that each update correlates its noise only with the immediately preceding iteration.

BISR provides the greatest benefit when larger bandwidth, and hence more memory, is available. DP-$\lambda$CGD is instead designed for the low-memory regime, achieving strong performance with the smallest nontrivial bandwidth, but at the cost of reduced utility from its very limited correlation window.
In the next section, we introduce a unified family that interpolates between these regimes through a tunable fractional power.


\section{$\gamma$-Banded Inverse Fractional Root (\method)}
\label{sec:gammabifr}
Our goal is to obtain an explicit banded-inverse factorization that performs well across a broad range of bandwidths, including the low-memory regime where BISR is less effective. The key observation is that the inverse square-root construction underlying BISR fixes the matrix exponent to $1/2$, although this choice need not be optimal when the bandwidth is small.

We replace the fixed square root by a tunable fractional power $\gamma\in[0,1]$. Varying $\gamma$ changes how aggressively previously added noise is canceled, allowing the factorization to adapt to the available bandwidth. The resulting family recovers BISR at $\gamma=1/2$ and, at bandwidth $p=2$, includes DP-$\lambda$CGD as a special case for $\gamma=\lambda$. We now formalize this construction, which we call \method.

\begin{definition}[{\method: \(\gamma\)-Banded Inverse Fractional Root}]
\label{def:gamma_bifr}
Fix \(0\le\gamma \le 1\) and an integer bandwidth parameter \(p\ge 1\).
We define the strategy \method matrix \(C_\gamma^{(p)}\) through its inverse by
\begin{equation}  
\bigl(C_\gamma^{(p)}\bigr)^{-1} := 
\begin{pmatrix} 
1 & 0 & \cdots & 0 & \cdots &0\\ 
(\tilde c_{\gamma})_1 & 1 & \cdots & 0 & \cdots &0 \\ 
\vdots & \vdots & \ddots & \vdots & \ddots & \vdots\\ 
(\tilde c_{\gamma})_{p - 1} & (\tilde c_{\gamma})_{p-2} & \cdots & 1 & \cdots &0\\
0 & (\tilde c_{\gamma})_{p - 1} & \cdots & (\tilde c_{\gamma})_1 &\cdots & 0 \\ 
\vdots & \vdots & \ddots & \vdots & \ddots & \vdots \\ 
0 & 0 & \cdots & (\tilde c_{\gamma})_{p - 1}&\cdots & 1 
\end{pmatrix}, 
\qquad \text{for} \quad (\tilde c_{\gamma})_i := (-1)^i \binom{\gamma}{i}, 
\end{equation}
where $\binom{\gamma}{i} = \prod_{j = 1}^{i} \frac{\gamma + 1 - j}{j}$ is a generalized binomial coefficient. That is, \(\bigl(C_\gamma^{(p)}\bigr)^{-1}\) is the lower triangular Toeplitz matrix
with  first \(p\) subdiagonals given by
\(1, (\tilde c_{\gamma})_1,\dots,(\tilde c_{\gamma})_{p - 1}\), and all remaining entries equal to \(0\).
\end{definition}

The correlation coefficients $\tilde{c}_{\gamma}$ can be computed efficiently using the simple linear recurrence
\[
(\tilde{c}_{\gamma})_0 = 1,
\qquad
(\tilde{c}_{\gamma})_j = (\tilde{c}_{\gamma})_{j-1} \cdot \frac{j - 1 - \gamma}{j}
\quad \text{for } j \geq 1.
\]
We refer to Lemma~\ref{lem:bounds-ctilde} in the appendix for a proof of the form of the coefficients $\tilde{c}_{\gamma}$, as well as for tight bounds on their values. In Figure~\ref{fig:rmse_comparison_non_amplified}, we show the behavior of various factorizations in the low-memory (small-$p$) regime, demonstrating that the proposed \method achieves almost optimal RMSE values while remaining an explicit factorization that is easy to construct and analyze.

\begin{figure}[t]
\centering
\begin{subfigure}[t]{0.32\textwidth}
    \centering
    \includegraphics[width=\textwidth]{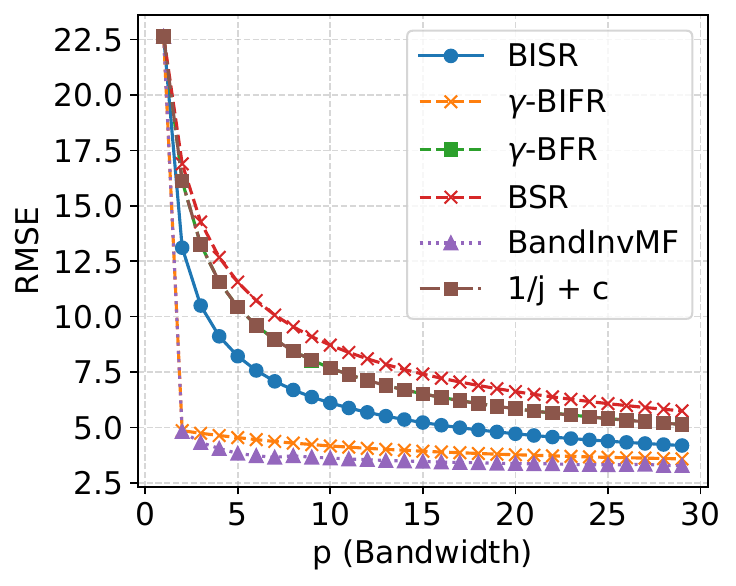}
    \caption{ $k=1$}
\end{subfigure}
\hfill
\begin{subfigure}[t]{0.32\textwidth}
    \centering
    \includegraphics[width=\textwidth]{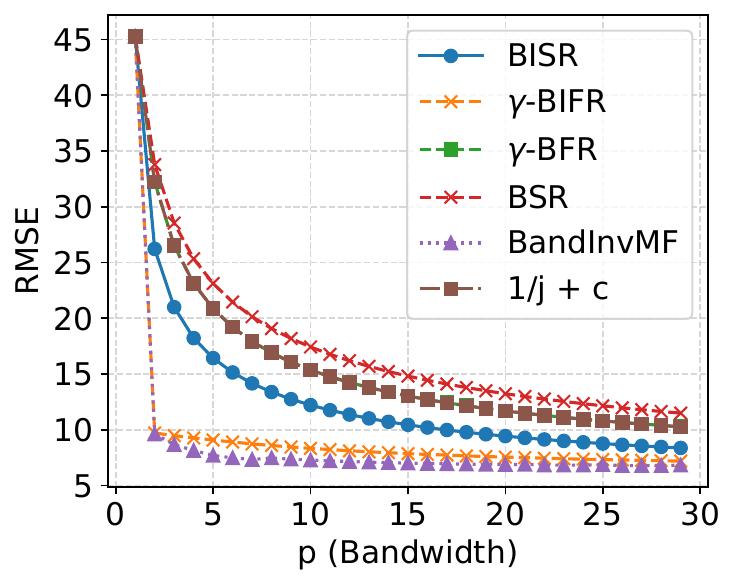}
    \caption{$k=4$}
\end{subfigure}
\hfill
\begin{subfigure}[t]{0.32\textwidth}
    \centering
    \includegraphics[width=\textwidth]{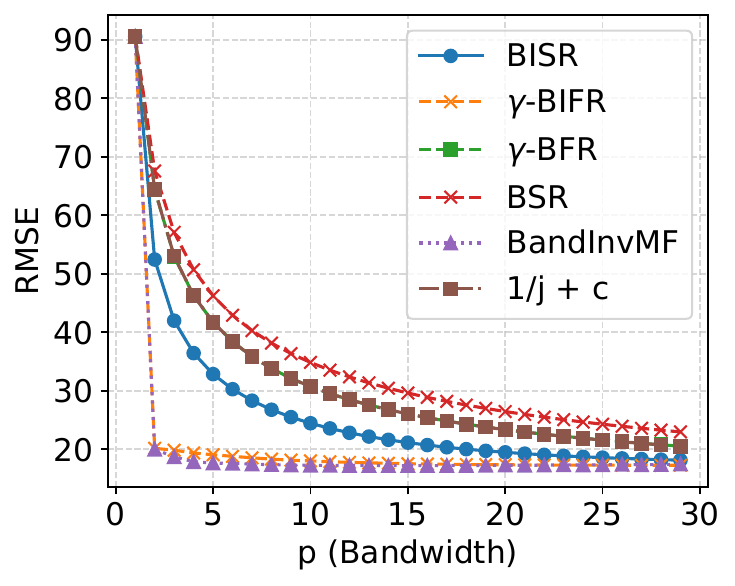}
    \caption{$k=16$}
\end{subfigure}
\caption{Comparison of the proposed $\gamma$-BIFR and $\gamma$-BFR factorizations with the explicit BSR~\citep{kalinin2024} and BISR factorizations, as well as the zeroth-order banded factorization $1/j + c$, where $c$ is optimized to minimize the RMSE. All plots use $n=1024$ iterations for different numbers of participations $k$ and no amplification by subsampling. We observe that $\gamma$-BIFR is consistently much closer to the optimal banded inverse factorization BandInvMF, substantially improving the error compared with BISR. In contrast, for small bandwidths, $\gamma$-BFR provides little improvement over the $1/j + c$ factorization.}
\label{fig:rmse_comparison_non_amplified}
\end{figure}

We theoretically bound the RMSE of the proposed \method in the following theorem.
\begin{theorem}\label{thm:bifr-multi-participation}
    For any bandwidth $2\le p\le n$ and number of participations (epochs) $1\le k\le\frac nb$,
    \method with $\gamma=\gamma^\star$ where
    \begin{equation*}
       \gamma^\star = \max\left\{\frac{1}{2}, 1-\frac{\log(1 + p\log p/\sqrt{n})}{\log p}  \right\}
    \end{equation*}
    attains an error bound of
    \begin{equation*}
        \mathcal{E}\bigl(B_{\gamma}^{(p)},C_{\gamma}^{(p)}\bigr)
        = O\!\left(\sqrt k\log p+\sqrt{\frac{nk}{b}}
        +\sqrt{\frac{nk\log(1+p/\sqrt n)}{p}}+\sqrt{\frac{kp\log p}{b}}\right).
    \end{equation*}
\end{theorem}

The proof of \Cref{thm:bifr-multi-participation} relies on the following lemma, which may also be of independent interest. It provides a tight, convenient characterization, up to constant factors, of the sensitivity of any Toeplitz matrix $C$ with positive, nonincreasing diagonals in terms of standard operator norms.

\begin{lemma}[Multi-participation sensitivity via operator norms]\label{lem:sens-toeplitz-bound}
Let $C \in \mathbb{R}^{n \times n}$ be a lower-triangular Toeplitz matrix with nonnegative, nonincreasing entries
$c_0 \ge c_1 \ge c_2 \ge \dots \ge c_{n-1} \ge 0$ on the diagonals.
Then the \emph{sensitivity} of $C$ in the setting of \emph{$b$-min-separation} satisfies the following bounds:
\begin{equation*}
    \frac{k}{2}\|C\|_{1 \to 2}^2 + \frac{k}{3b}\|C\|_{1 \to 1}^2
    \le \mathrm{sens}_{k,b}(C)^2
    \le k\|C\|_{1 \to 2}^2 + \frac{k}{b}\|C\|_{1 \to 1}^2,
\end{equation*}
where the lower bound requires $n = kb$, while the upper bound holds for any $k \le \frac{n}{b}$. The norms $\|\cdot\|_{1\to 2}$ and $\|\cdot\|_{1\to 1}$ denote the maximum column $\ell_2$ and $\ell_1$ norms, respectively.
\end{lemma}

\paragraph{Comparison with BISR and DP-$\lambda$CGD.}

The bound in \Cref{thm:bifr-multi-participation} coincides with the RMSE guarantee of BISR, except that the $\log p$ in the term $\sqrt{nk\log(p)/p}$ is replaced by $\log(1+p/\sqrt{n}) \leq 2\log p$.
In particular, \method is never worse than BISR, and exactly recovers it for $\gamma=1/2$.
This guarantee is asymptotically optimal for sufficiently large bandwidth: when $k=\lceil n/b\rceil$, the optimal RMSE over all factorizations is $\Omega(k+\sqrt{k}\log n)$, and BISR attains this rate with an appropriate choice of $p=O(b\log b)$.

At small bandwidths, however, the square-root choice $\gamma=1/2$ can be substantially suboptimal. In particular, for $p=2$ and $k=\lceil n/b\rceil$, the BISR bound scales as $O(k+\sqrt{nk})$, whereas DP-$\lambda$CGD, which is designed specifically for bandwidth $p=2$, achieves $\Theta(k+\sqrt{k}n^{1/4})$. The gap is especially pronounced when the number of participations $k$ is small.

\Cref{thm:bifr-multi-participation} extends this favorable low-bandwidth behavior beyond $p=2$.
In particular, for $p\leq \sqrt{n}$, an appropriate choice of $\gamma$ matches the asymptotic scaling of DP-$\lambda$CGD while allowing larger bandwidths.
Thus, \method retains the large-bandwidth behavior of BISR while extending the favorable low-bandwidth behavior of DP-$\lambda$CGD.

\begin{algorithm}[t!]
\caption{Differentially Private Model Training with \method}
\label{alg:method}
\begin{algorithmic}[1]
\Require Model initialization $\theta_0 \in \mathbb{R}^d$, dataset $\mathcal{D}$, batch size $B$, clip norm $\zeta$, learning rate $\eta>0$, loss $\ell(\theta,d)$, parameter $\gamma \in [0,1)$,  $b$ separation, bandwidth $p \ge 1$, target privacy level $\varepsilon,\delta$, number of iterations $n$, pseudorandom noise generator $\mathrm{Gen}$, random seed $\mathrm{GSt}_0$.
\State Let $(\tilde{c}_{\gamma})_0 =(c_{\gamma}^{(p)})_0 =1$
\State $(\tilde{c}_{\gamma})_j  \gets (\tilde{c}_{\gamma})_{j - 1} \cdot \frac{j - 1 - \gamma}{j} \qquad \text{for} \quad j = 1, \dots p - 1$
\State $(c_{\gamma}^{(p)})_j \gets -\sum\limits_{i = 1}^{\min(j, p - 1)}(\tilde{c}_{\gamma})_{i}(c_{\gamma}^{(p)})_{j - i}\qquad \text{for} \quad j = 1, \dots n - 1$
\State $C^{(p)}_{\gamma} \gets \mathrm{LTT}(c_{\gamma}^{(p)})$ \Comment{Construct the strategy matrix of \method}
\State  $\sigma \gets \mathrm{Subsampling\_Accountant}(C^{(p)}_{\gamma}, n, b,  B, |\mathcal{D}|, \varepsilon, \delta)$ \Comment{Compute noise multiplier}

\For{$i=1$ to $n$}
\State Sample batch $S_i$ in accordance with the sampling scheme
\State $g_j \gets \nabla_\theta \ell(\theta_{i-1}, d_j) \qquad \text{for} \quad j=1, \dots, |S_i|$
    \State $\tilde g_j \gets \min\!\bigl(1, \frac{\zeta}{\|g_j\|}\bigr)\, g_j \qquad \text{for} \quad j=1, \dots, |S_i|$ \Comment{per-example clipping}
    
    \State $x_i \gets \sum_{j=1}^{|S_i|} \tilde g_j$ \Comment{aggregate clipped gradients}
    \State $\mathrm{Gen} \gets \mathrm{Gen.set\_state}(\mathrm{GSt}_{\max(i - p, 0)})$ \Comment{Initialize PRNG's state}
    \State $\hat x_i \gets x_i$
    \For{$j = \max(1, p - i + 1)$ to $p - 1$}
    \State $Z_{i-p + j} \sim \mathrm{Gen}()$ \Comment{regenerate noise}
    \State $\hat x_i \gets \hat{x}_i + \zeta\sigma (\tilde{c}_{\gamma})_{p - j} Z_{i-p + j}$
    \State $\mathrm{GSt}_{i - p + j} \gets \mathrm{Gen.get\_state()}$
    \EndFor
    \State $Z_i \sim \mathrm{Gen}()$ \Comment{generate fresh noise}
    \State $\mathrm{GSt}_{i} \gets \mathrm{Gen.get\_state()}$
    \State $\hat x_i \gets \hat x_i + \zeta\sigma Z_i$
    \State $\theta_i \gets \theta_{i-1} - \frac{\eta}{B}\hat x_i$ \Comment{model update}
\EndFor
\State \Return $\theta_n$
\end{algorithmic}
\end{algorithm}

\paragraph{Noise Regeneration.}
As a factorization in the class of banded inverse methods, \method can efficiently correlate noise without storing it in a buffer by using pseudorandomness. Most noise generators, including cryptographically secure ones, produce noise sequentially and deterministically from a potentially physically random (system-random) seed. Thus, one can store the state of the generator, for example, a 128-bit state in the commonly used Philox4x32-10 generator in PyTorch, and use it to \emph{regenerate} the noise from the last $p - 1$ iterations, assuming bandwidth $p$. Therefore, the proposed method can be implemented with negligible memory overhead for noise storage compared with DP-SGD.

Noise regeneration, however, incurs additional computation time. Regenerating noise from all previous iterations would introduce substantial overhead, making non-banded-inverse methods such as BandMF, BSR, and BLT impractical in this setting. In the work of \citet{kalinin2026dplambdacgd}, it was shown that  $16$ noise vectors can be regenerated without significant time overhead relative to DP-SGD. Moreover, this limit can be pushed even further by increasing the batch size.

\paragraph{Banded factorization.}
We also compare against the following banded alternative.
\begin{definition}[{$\gamma$-BFR: $\gamma$-Banded Fractional Root}]
\label{def:gamma_bfr}
Fix $0\le \gamma<1$ and an integer bandwidth parameter $p\ge 1$.
The $\gamma$-BFR strategy matrix $C_{\gamma,\mathrm{BFR}}^{(p)}$ is the bandwidth-$p$ truncation of $E^\gamma$, retaining the main diagonal and the first $p-1$ subdiagonals.
\end{definition}

For non-amplified RMSE, $\gamma$-BFR scales poorly with bandwidth and does not improve over the explicit banded factorization of \citet{scalingmckenna2024}, whose coefficients have the form $1/j+c$ for a constant $c$ (see Appendix J therein). Nevertheless, $\gamma$-BFR is useful in the amplified setting without memory restrictions, where it improves RMSE under $b$-sep subsampling; see Section~\ref{sec:amplified_rmse}.

\subsection{Amplification by Subsampling}
\label{sub:amplification_by_subsampling}
\begin{figure}[t!]
    \centering

    \begin{subfigure}[t]{0.46\linewidth}
        \centering
        \includegraphics[width=\linewidth]{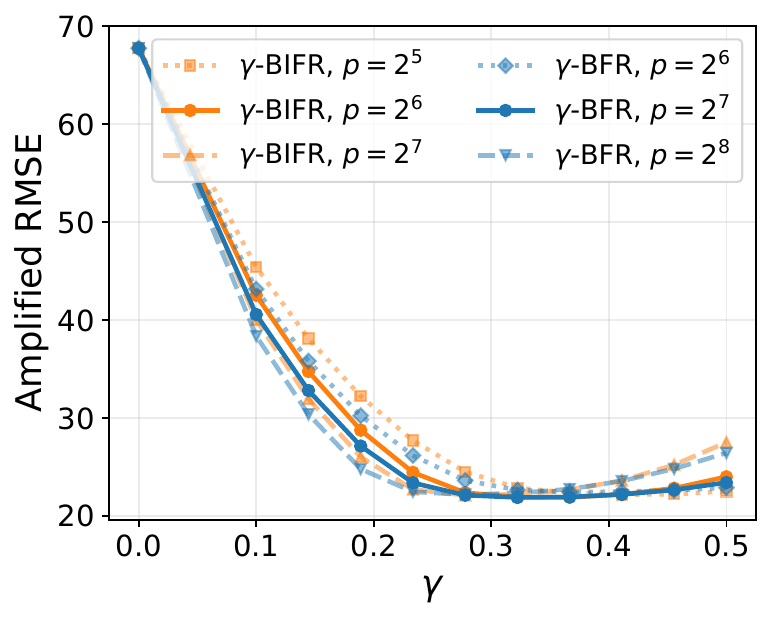}
        \caption{Amplified RMSE vs $\gamma$}
        \label{fig:bifr_amplified_rmse_gamma}
    \end{subfigure}
    \hfill
    \begin{subfigure}[t]{0.46\linewidth}
        \centering
        \includegraphics[width=\linewidth]{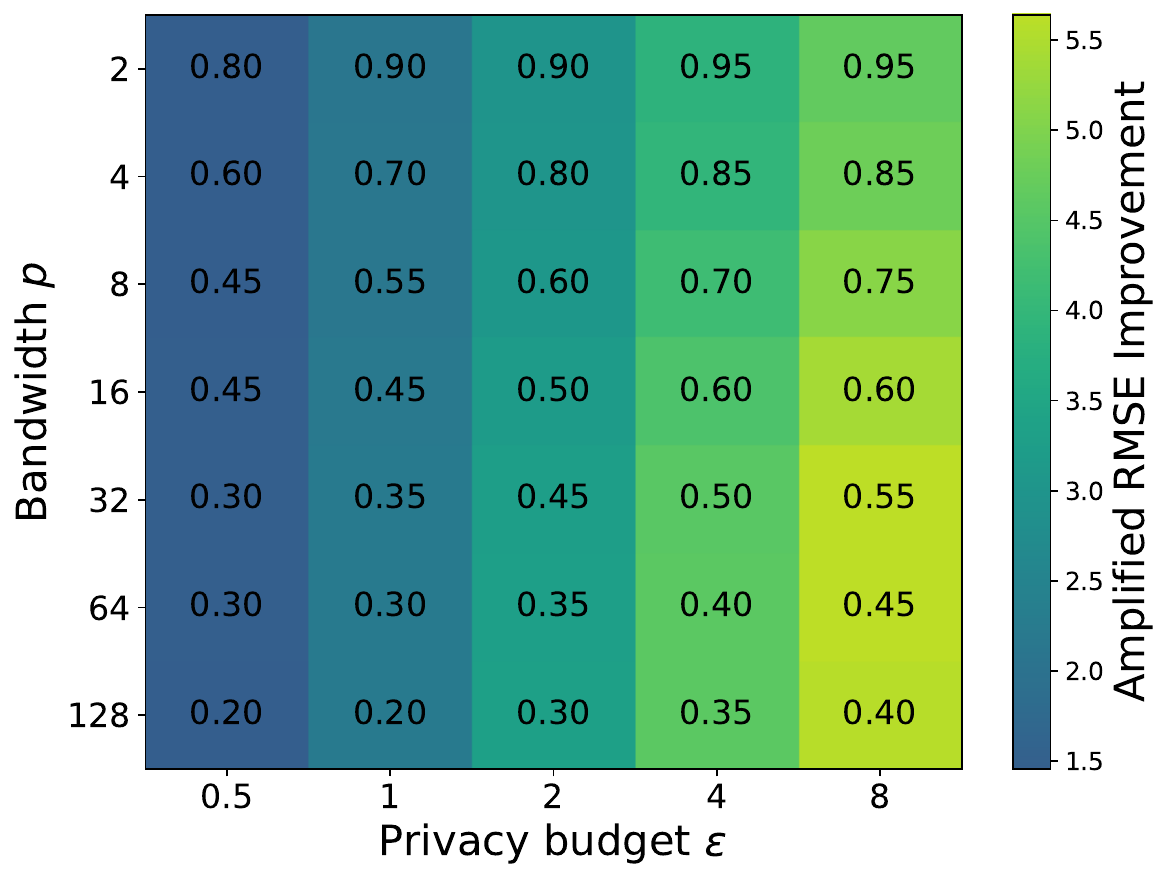}
        \caption{Optimal $\gamma$ (text) and Amplified RMSE (colors)}
        \label{fig:bifr_optimal_gamma}
    \end{subfigure}
\caption{%
(a) Balls-in-Bins amplified RMSE for the proposed $\gamma$-BIFR and $\gamma$-BFR methods with $n = 2048$, $k = 8$, $\varepsilon = 1$, and $\delta = 10^{-5}$.
(b) Optimal values of $\gamma$ for $\gamma$-BIFR are shown in the cells, together with the corresponding optimal amplified RMSE improvement over DP-SGD ($\gamma = 0$), indicated by color. The results are reported for CIFAR-10/100 training setting with batch size $512$, $k = 20$ epochs, and $\delta = 10^{-5}$.
}
    \label{fig:bifr_amplified_rmse}
\end{figure}

Matrix mechanisms can also be analyzed together with amplification by subsampling, meaning that the privacy guarantees can be improved by accounting for randomness in batch formation. Several sampling schemes are available for matrix mechanisms, including ``Cyclic Poisson'' \citep{choquette2023amplified}, ``Poisson'' subsampling \citep{choquette2023privacy}, Balls-in-Bins \citep{choquette2024near, schuchardt2026sampling}, and the novel $b$-min-separated Balls-in-Bins \citep{dong2026privacy}. Each sampling scheme has its own limitations. ``Cyclic Poisson'' provides limited amplification and is applicable only to banded methods. ``Poisson'' accounting based on conditional composition is extremely loose. Balls-in-Bins, which is the most practical method for banded-inverse matrices and is available in the JAX\_privacy library \citep{mckenna2026jax}, provides, however, limited amplification because batch formation involves less randomness than in Poisson subsampling. Finally, the analysis of $b$-min-separated Balls-in-Bins was designed for banded matrices and is less applicable to banded-inverse matrices.

In this subsection, we analyze the amplified RMSE under Balls-in-Bins amplification by subsampling, as it provides the strongest amplification for banded-inverse methods. Specifically, instead of computing the multi-participation sensitivity $\mathrm{sens}_{k,b}(C)$, we directly plug the matrix into the accountant and scale the expected error by the resulting noise multiplier; see \Cref{eq:amplified_rmse}. First, we observe that the optimal value of $\gamma$ for minimizing the amplified RMSE is smaller than $1/2$ when arbitrary bandwidth is allowed; see Figure~\ref{fig:bifr_amplified_rmse}. In the left subfigure (a), we also observe that, near the optimal value, the objective is not very sensitive to the choice of $\gamma$. Moreover, the alternative $\gamma$-BFR factorization, being banded, achieves almost identical performance, when there is no restriction on the bandwidth. In the right subfigure (b), we evaluate the optimal value of $\gamma$ 
(shown as text) and the improvement in amplified RMSE over DP-SGD 
($\gamma = 0$) for the corresponding privacy level (shown as colors). We show that for 
small bandwidth values, the optimal value of $\gamma$ is much larger than 
$1/2$, whereas for large bandwidth values, it is lower. The optimal RMSE 
is insensitive to the bandwidth when optimized over $\gamma$ for small 
$\varepsilon$. However, for large values of $\varepsilon$, much larger values of 
$\gamma$ are preferred, with $p = 64$ being optimal in the considered 
practical setting.

We compare the amplified RMSE of different factorizations in Table~\ref{tab:amplified_rmse} in the appendix. We show that, in the high-privacy regime, the proposed \method matches or improves upon the amplified RMSE of competing methods. We also note that zero-order methods, which require optimizing only a single scalar parameter, achieve the best performance, outperforming the gradient-based optimization of BandMF with respect to the Balls-in-Bins accountant described in \citep{choquette2024near}. We emphasize that optimizing a factorization through the Balls-in-Bins analysis is a challenging optimization problem and offers no guarantees of convergence to a globally optimal factorization. Together with the $1/j + c$ factorization, which was proposed but whose amplified RMSE was never evaluated, our factorizations are the first to demonstrate improved amplified RMSE performance. This metric also better reflects practical performance, as it accounts for the amount of noise added to the model. For the evaluation, we used Google DeepMind’s JAX\_privacy implementation \citep{mckenna2026jax} of various factorizations and amplification schemes.

\begin{figure}[t!]
    \centering
    \begin{subfigure}[t]{0.45\linewidth}
        \centering
        \includegraphics[width=\linewidth]{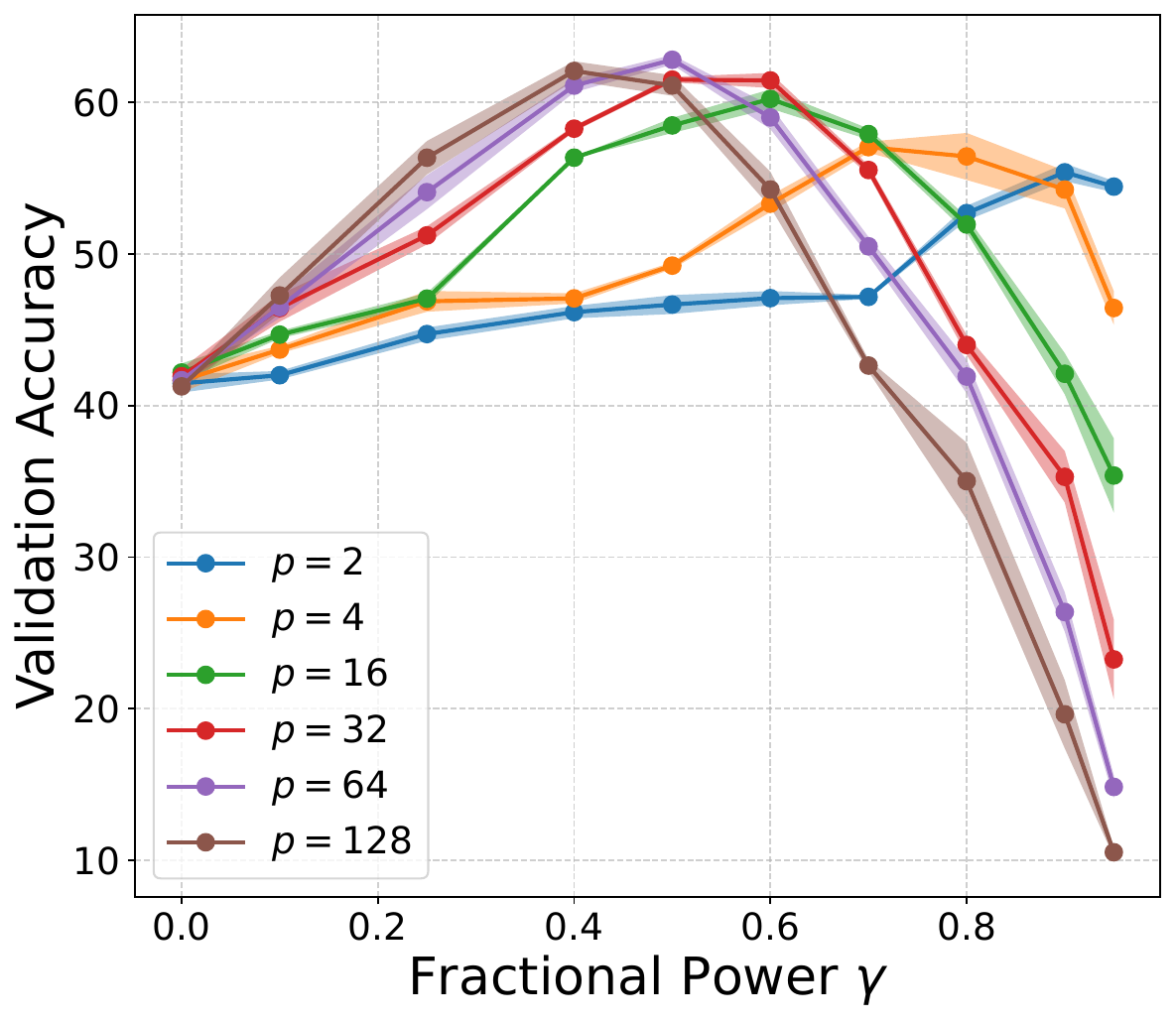}
        \caption{privacy budget $\varepsilon = 8$.}
    \end{subfigure}
    \hfill
    \begin{subfigure}[t]{0.45\linewidth}
        \centering
        \includegraphics[width=\linewidth]{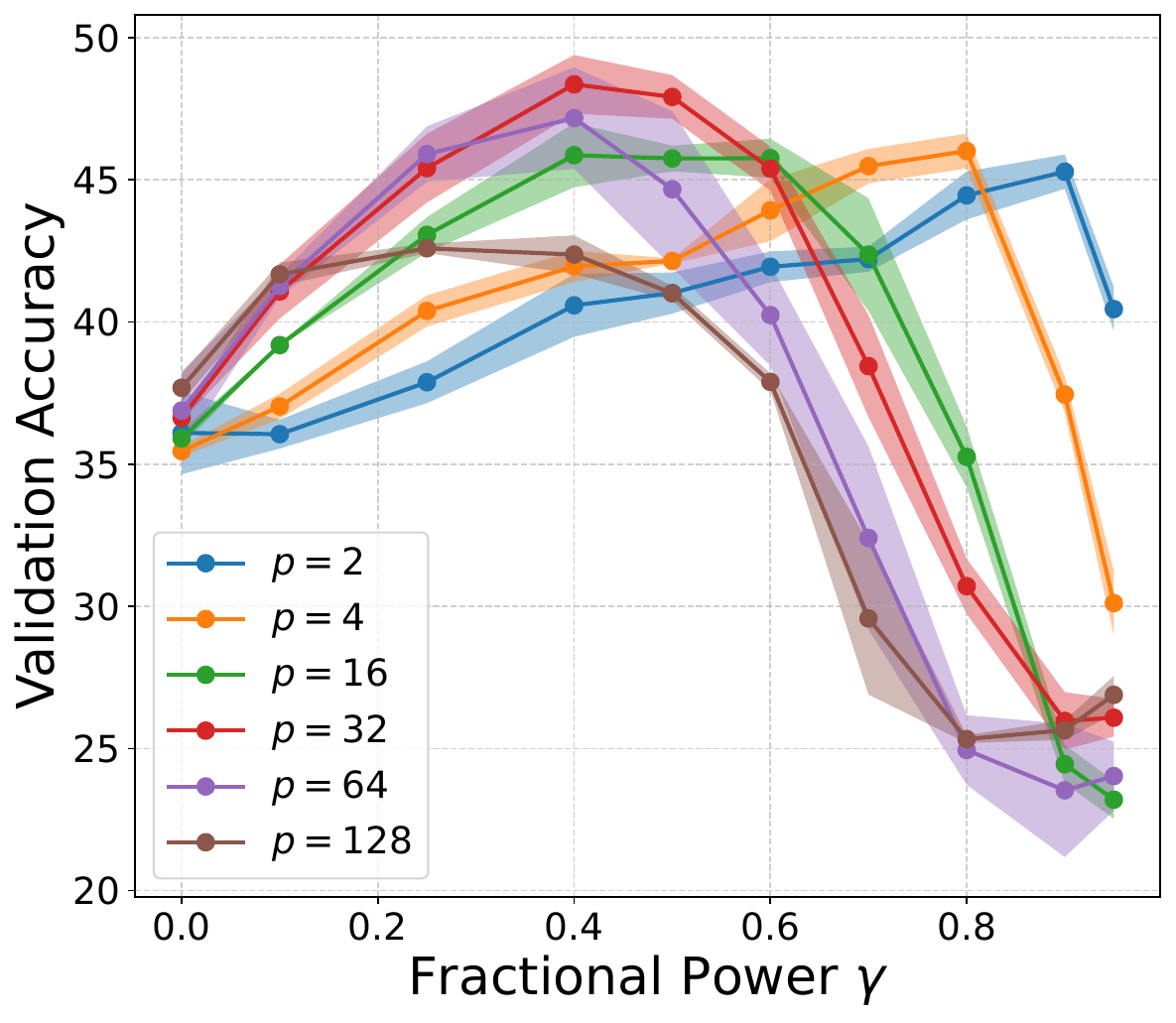}
        \caption{privacy budget  $\varepsilon = 1$.}
    \end{subfigure}
    \caption{Validation accuracy on CIFAR-10 for different values of $\gamma$ using Balls-in-Bins amplified $\gamma$-BIFR, shown for different bandwidths $p$, with $k = 10$ epochs and batch size $B = 128$, $\delta=10^{-5}$. For small bandwidths, choosing $\gamma$ closer to $1$ rather than the BISR choice $\gamma = 1/2$ can significantly improve performance. However, for larger bandwidths, $\gamma$ lower than $1/2$ could be the best choice.} 
    \label{fig:cifar10_lambda_bifr_amplified}
\end{figure}

\section{Experiments}
\label{sec:experiments}

In this section, we evaluate the performance of the proposed \method on machine learning tasks. We begin by studying how performance depends on the value of $\gamma$. Specifically, we consider a simple setting of training a CNN model on the CIFAR-10 dataset with batch size $128$ over $10$ epochs at various privacy levels. In Figure~\ref{fig:cifar10_lambda_bifr_amplified}, we plot the validation accuracy as a function of $\gamma$. The results show that, for small bandwidth $p$, the optimal value of $\gamma$ differs significantly from $1/2$, and thus outperforms the BISR factorization in the small-bandwidth regime. Moreover, the accuracy achieved with bandwidth $p > 2$ is substantially higher than that with bandwidth $p = 2$, thereby outperforming DP-$\lambda$CGD.

We compare the test accuracy of \method with prior methods under two memory
regimes: a high-memory setting with bandwidth $p=64$ for all banded methods and
a low-memory setting with $p=4$.
We evaluate image classification on CIFAR-10 and CIFAR-100 using the 3-block VGG model from prior work~\citep{kairouz2021practical}, and sentiment analysis on IMDb by fine-tuning BERT-Base~\citep{devlin2019bert}. The results for CIFAR-10 are shown in Figure~\ref{fig:bifr_fixedp_lowmem}, while the results for CIFAR-100 and IMDb are shown in \Cref{fig:bifr_fixedp_cifar100} and \Cref{fig:bifr_fixedp_imdb_bert}, respectively, in the appendix.
The CIFAR-10/100 experiments are run for 20 epochs, whereas the IMDb experiments are run for 10 epochs.
In all cases, we use a batch size of $512$, clipping norm $10$, privacy budgets $\varepsilon \in \{0.5,1,2,4,8\}$, and $\delta=10^{-5}$.
We report final test accuracy with error bars computed over three repeats.

For the baselines, we compare against BLT~\citep{hasslefree2024}, BandInvMF and BISR~\citep{kalinin2025back}, BandMF~\citep{scalingmckenna2024}, and BSR~\citep{kalinin2024}.
We also compare against standard DP-SGD, for which we use the Opacus implementation~\citep{opacus} with Poisson subsampling and the PRV accountant~\citep{gopiNumericalCompositionDifferential2021}.
For the matrix-factorization baselines, we use Balls-in-Bins amplification with $500$k Monte Carlo samples, following the accountant of \citet{choquette2024near}.
For BLT, we use four buffers and we also include the DP-$\lambda$CGD comparison~\citep{kalinin2026dplambdacgd}. The results on CIFAR-10 and CIFAR-100 show that \method improves over the alternative methods in the
low-privacy regime and performs on par with the banded-inverse methods, with only
BLT achieving better performance in the high-privacy regime. The results on IMDb show comparable performance across the different methods.

\begin{figure}[t!]
    \centering
    \begin{subfigure}[t]{0.46\linewidth}
        \centering
        \includegraphics[width=\linewidth]{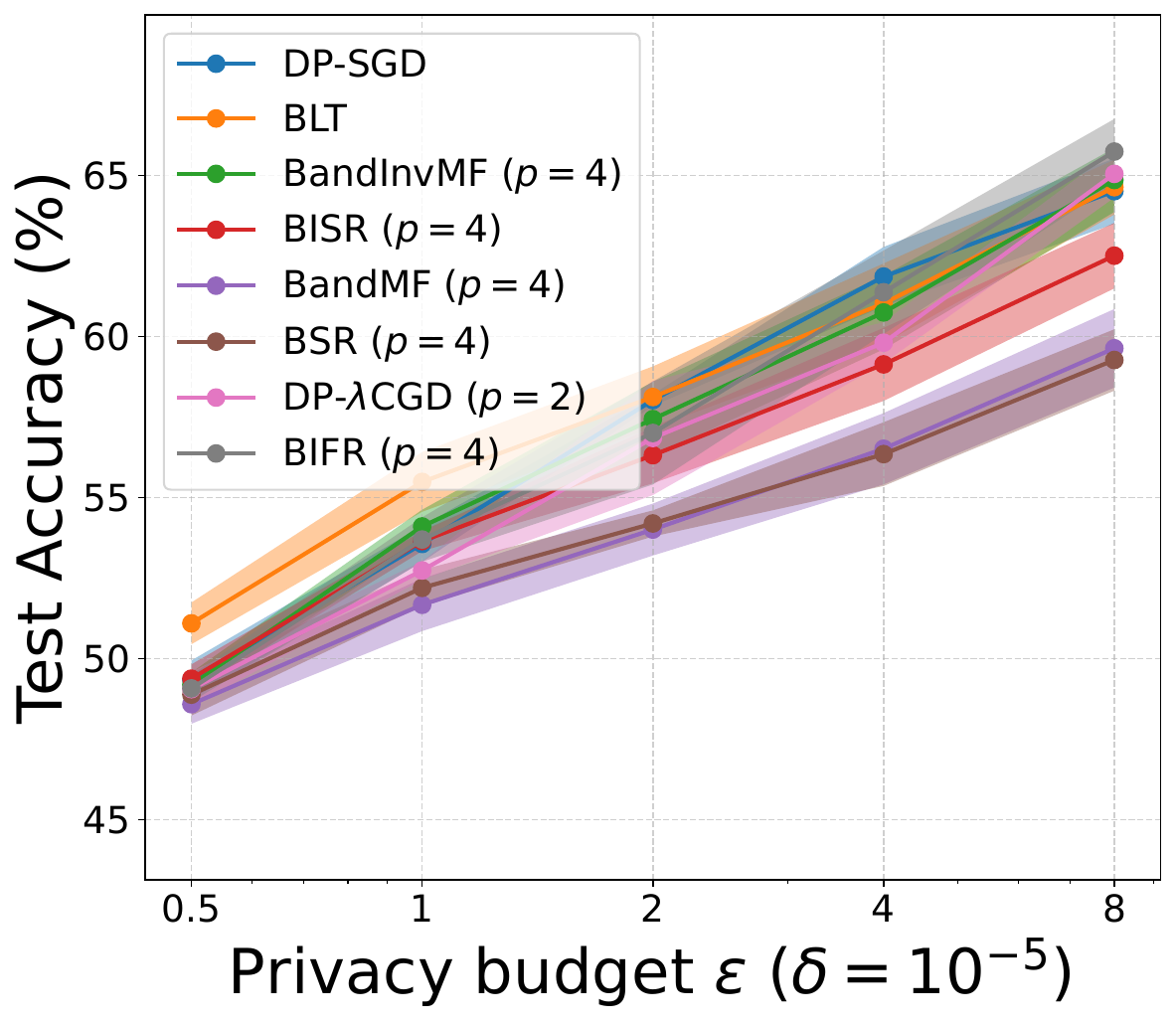}
        \caption{CIFAR-10/VGG. low memory}
        \label{fig:bifr_fixedp_lowmem_cifar10_vgg}
    \end{subfigure}
    \hfill
    \begin{subfigure}[t]{0.46\linewidth}
        \centering
        \includegraphics[width=\linewidth]{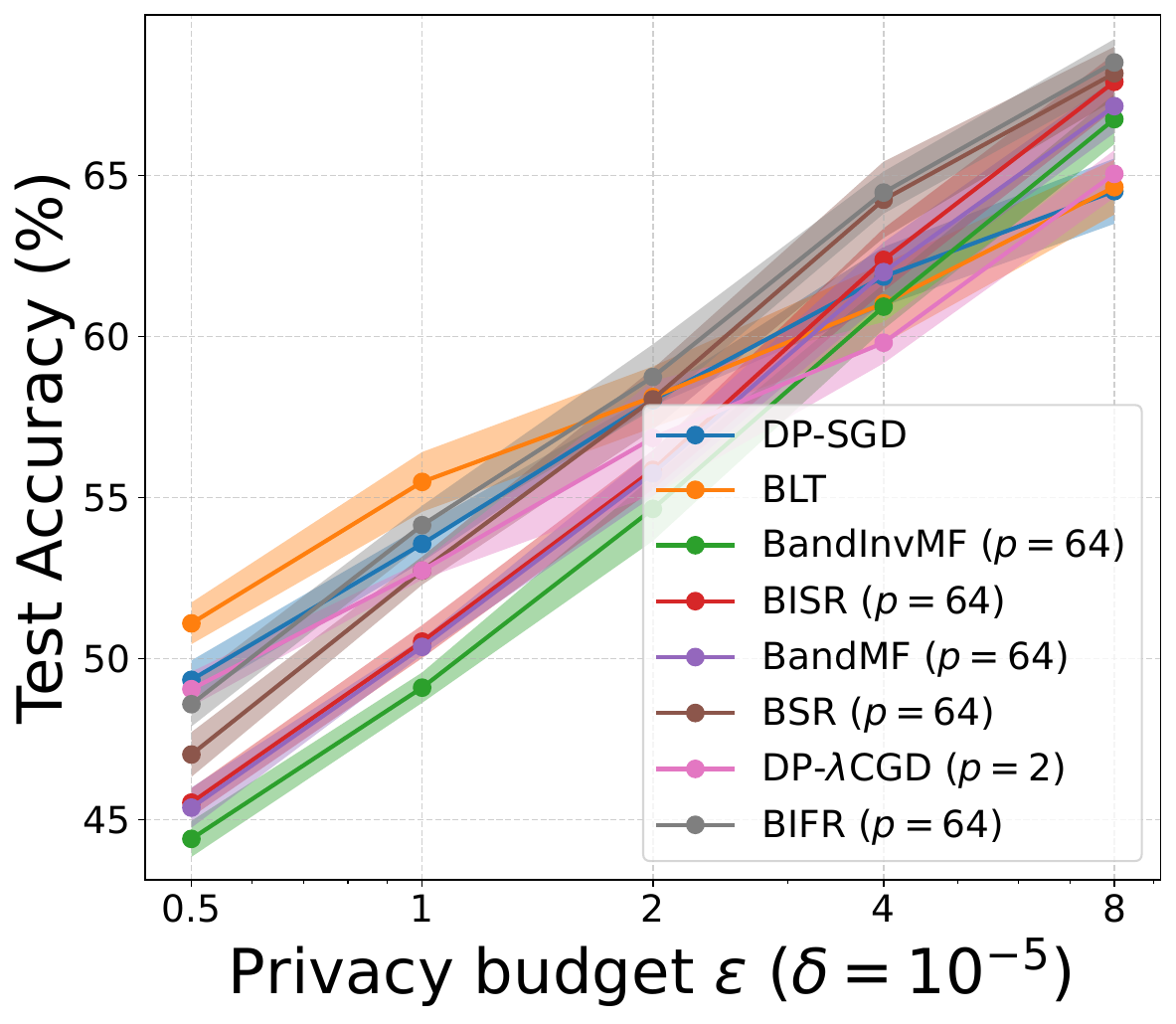}
        \caption{CIFAR-10/VGG. high memory}
        \label{fig:bifr_fixedp_highmem_cifar10_vgg}
    \end{subfigure}
    \caption{Test accuracy on CIFAR-10 in the low- ($p=4$) and high-memory ($p=64$) regimes for different factorizations across privacy budgets $\varepsilon$. The full experimental setup, including hyperparameters and a table of accuracy values (\Cref{tab:bifr-accuracy-ci-highmem-cifar10-vgg,tab:bifr-accuracy-ci-lowmem-cifar10-vgg}), is provided in Section~\ref{sec:experiment_details} of the appendix.
    }
    \label{fig:bifr_fixedp_lowmem}
\end{figure}

\section*{Conclusion and Future Directions}

We proposed a new factorization method that strictly generalizes both BISR and DP-$\lambda$CGD. We showed theoretically that it achieves RMSE at least as good as the better of the two methods, and we demonstrated numerically that it improves RMSE, amplified RMSE, and accuracy in CIFAR-10 and CIFAR-100 training. The proposed factorization can also be combined efficiently with noise regeneration, allowing it to run with negligible overhead compared with DP-SGD.

Matrix factorization techniques would benefit greatly from more accurate and efficient amplification accounting. As future work, it is important to combine banded-inverse factorizations with Poisson subsampling in a computationally efficient manner.

\subsection*{AI use statement}

In this work, we used generative AI tools to discover the statement of Lemma~\ref{lem:sens-toeplitz-bound} and the general idea underlying its proof. In particular, this suggestion was made by GPT Codex 5.5 for the upper bound and GPT 6 Astra for the lower bound. We have not used generative AI tools to generate the proofs of technical statements; the detailed proof of Lemma~\ref{lem:sens-toeplitz-bound}, as well as all other technical proofs, was developed by the authors.
Additionally, we used generative AI tools for editing and basic coding support. We have reviewed all AI-assisted work. In particular, the AI-suggested statement and proof idea were independently checked and developed by the authors, and any AI-assisted code was reviewed and verified for correctness. We take responsibility for the final content of this work, including text, claims, or artifacts produced with the aid of generative AI.

\section*{Acknowledgement}
We thank Ryan McKenna, Arun Ganesh, and Brendan McMahan for fruitful discussions on an earlier version of the paper. We thank Jalaj Upadhyay for his valuable feedback on the paper’s proofs. 

Nikita Kalinin: This work is
supported in part by the Austrian Science Fund (FWF) [10.55776/COE12]. 

Aki Rehn and Antti Honkela: This work was supported by the Research Council of Finland (Flagship programme: Finnish Center for Artificial Intelligence, FCAI, Grant 356499 and Grant 359111), the Strategic Research Council at the Research Council of Finland (Grant 358247) as well as the European Union (Project 101070617).
Views and opinions expressed are however those of the author(s) only and do not necessarily reflect those of the European Union or the European Commission.
Neither the European Union nor the granting authority can be held responsible for them.
The authors wish to thank the CSC – IT Center for Science, Finland for supporting this project with computational and data storage resources.
We acknowledge CSC (Finland) for awarding this project access to the LUMI supercomputer, owned by the EuroHPC Joint Undertaking, hosted by CSC (Finland) and the LUMI consortium.
The authors acknowledge the research environment provided by ELLIS Institute Finland.

Joel Daniel Andersson: Funded by the European union.
Views and opinions expressed are however those of the author(s) only and do not necessarily reflect those of the European Union or the European Research Council Executive Agency.
Neither the European Union nor the granting authority can be held responsible for them.
This project has received funding from the European Research Council (ERC) under the European Union's Horizon 2020 research and innovation programme (MoDynStruct, No. 101019564) \includegraphics[width=0.9cm]{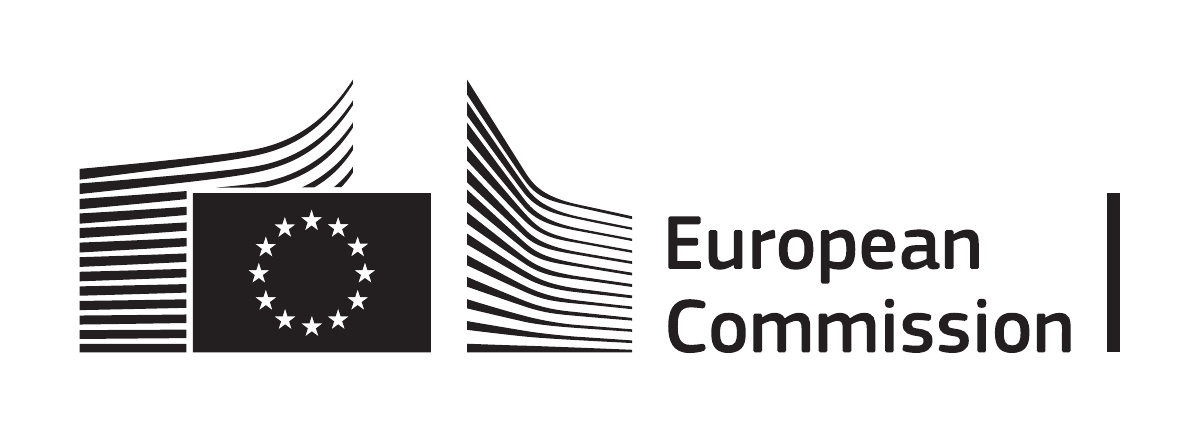}.

\bibliography{reference}
\bibliographystyle{iclr2027_conference}

\appendix

\clearpage

\section{Full Experimental Details}
\label{sec:experiment_details}

\begin{figure}[ht!]
    \centering
    \begin{subfigure}[t]{0.46\linewidth}
        \centering
        \includegraphics[width=\linewidth]{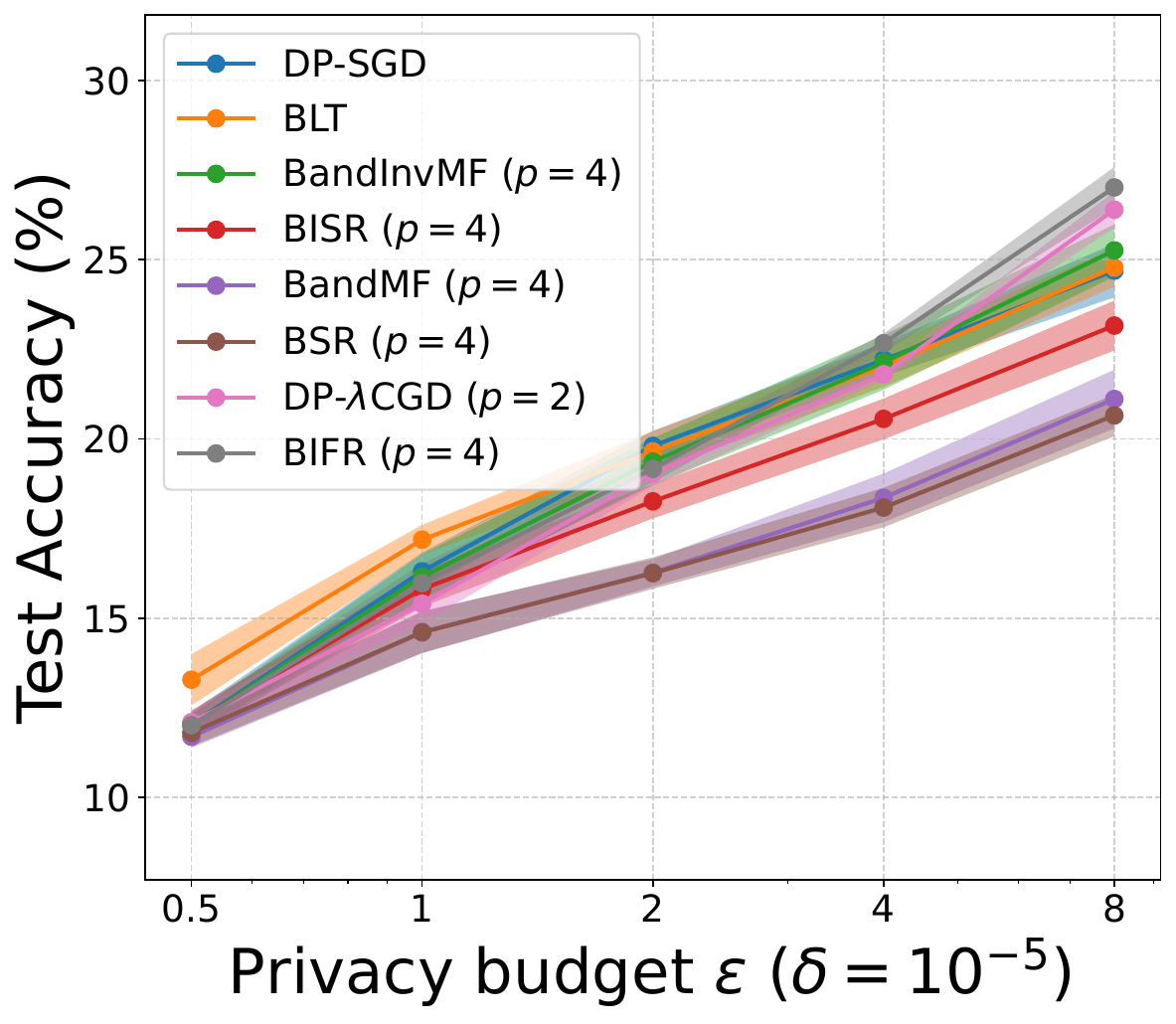}
        \caption{CIFAR-100/VGG. low memory}
        \label{fig:bifr_fixedp_lowmem_cifar100_vgg}
    \end{subfigure}
    \hfill
    \begin{subfigure}[t]{0.46\linewidth}
        \centering
        \includegraphics[width=\linewidth]{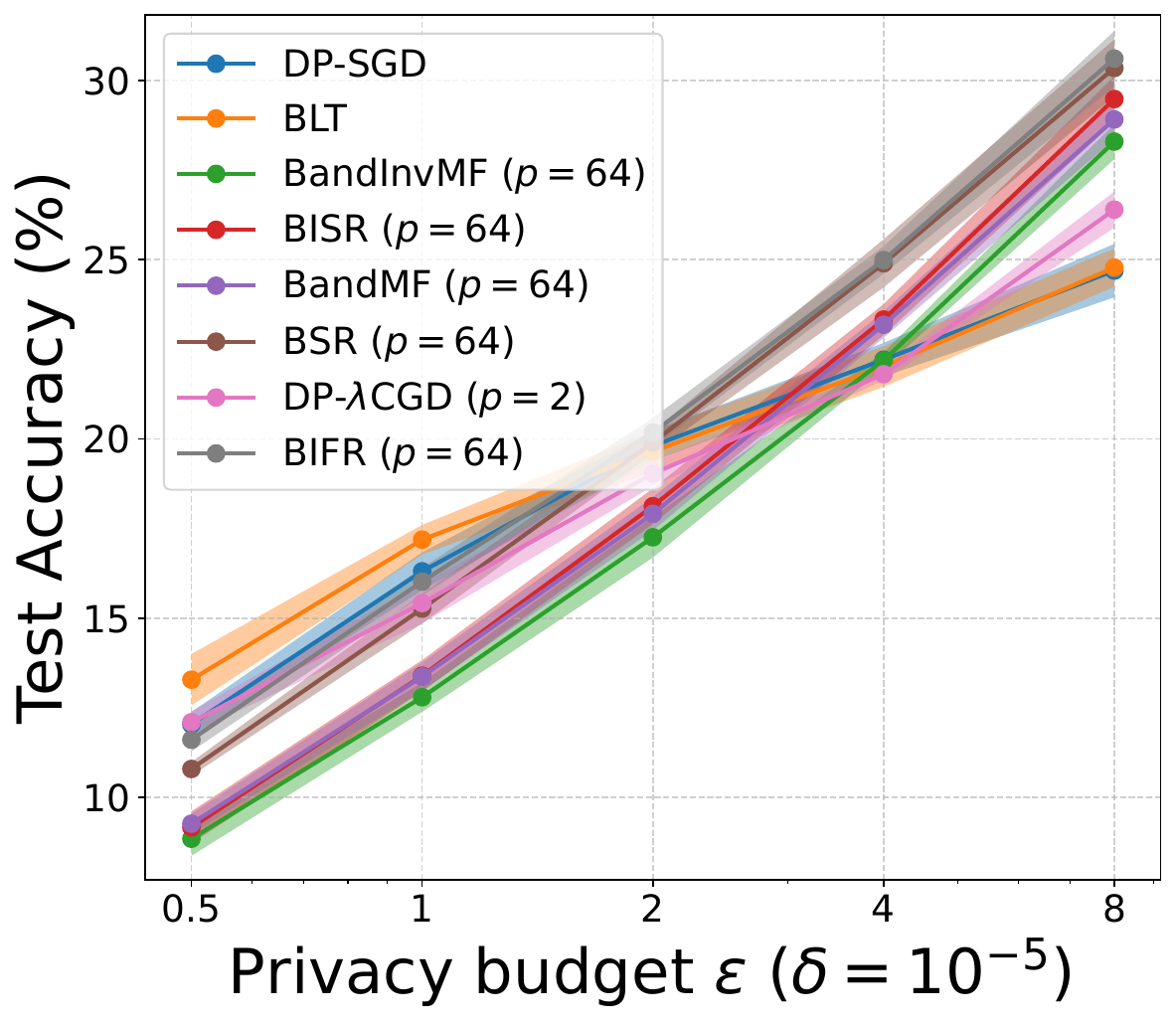}
        \caption{CIFAR-100/VGG. high memory}
        \label{fig:bifr_fixedp_highmem_cifar100_vgg}
    \end{subfigure}
    \caption{Test accuracy on CIFAR-100 in the low- ($p=4$) and high-memory ($p=64$) regimes for different factorizations across privacy budgets $\varepsilon$. The full experimental setup, including hyperparameters and a table of accuracy values (\Cref{tab:bifr-accuracy-cifar100-vgg-lowmem,tab:bifr-accuracy-cifar100-vgg-highmem}), is provided in Section~\ref{sec:experiment_details} of the appendix.
    }
    \label{fig:bifr_fixedp_cifar100}
\end{figure}

\begin{figure}[h]
    \centering
    \begin{subfigure}[t]{0.45\linewidth}
          \includegraphics[width=\linewidth]{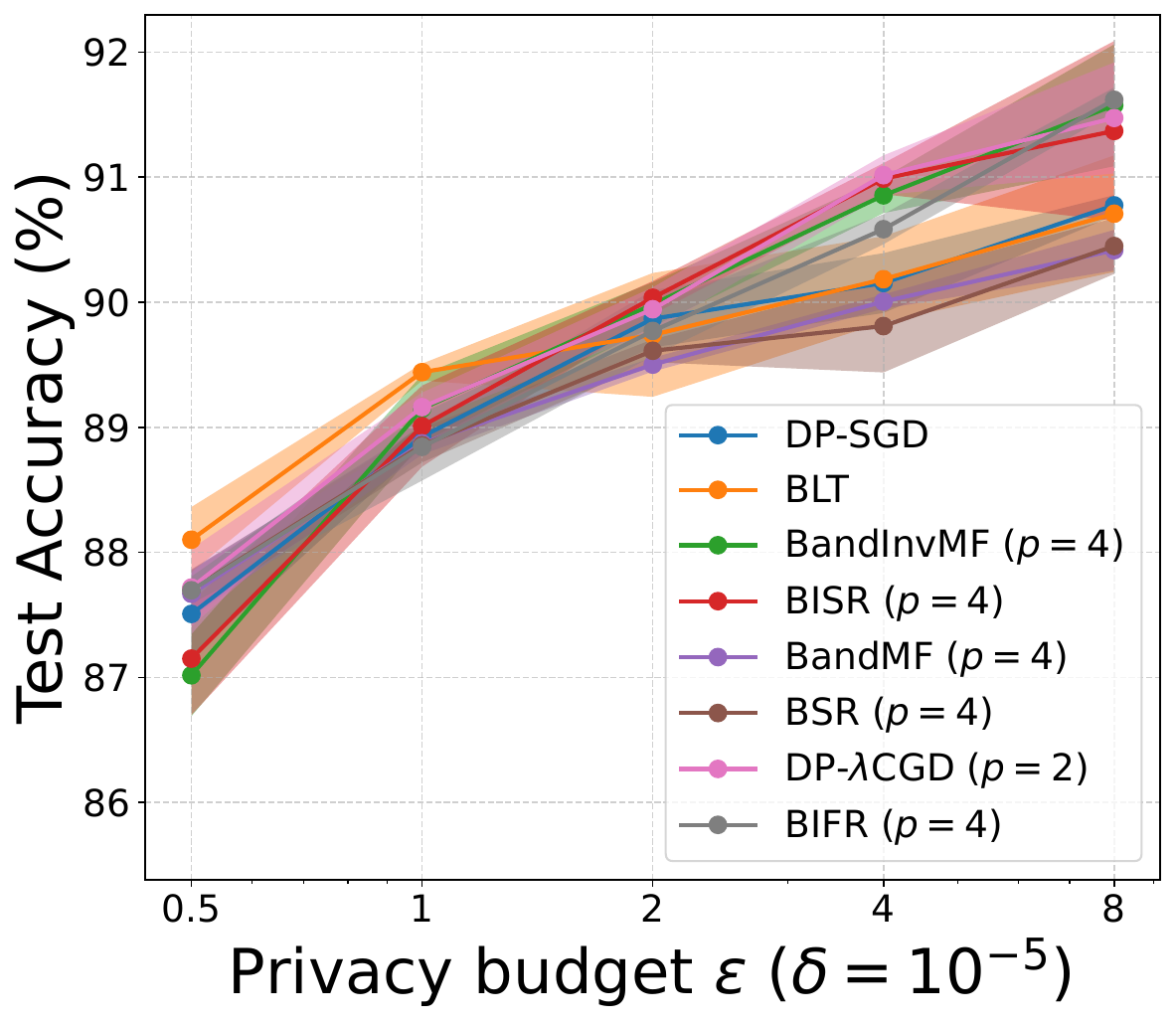}
        \caption{IMDb/BERT-Base. low memory}
        \centering
        \label{fig:bifr_fixedp_highmem_imdb_bert}
    \end{subfigure}
    \hfill
    \begin{subfigure}[t]{0.45\linewidth}
        \centering
        \includegraphics[width=\linewidth]{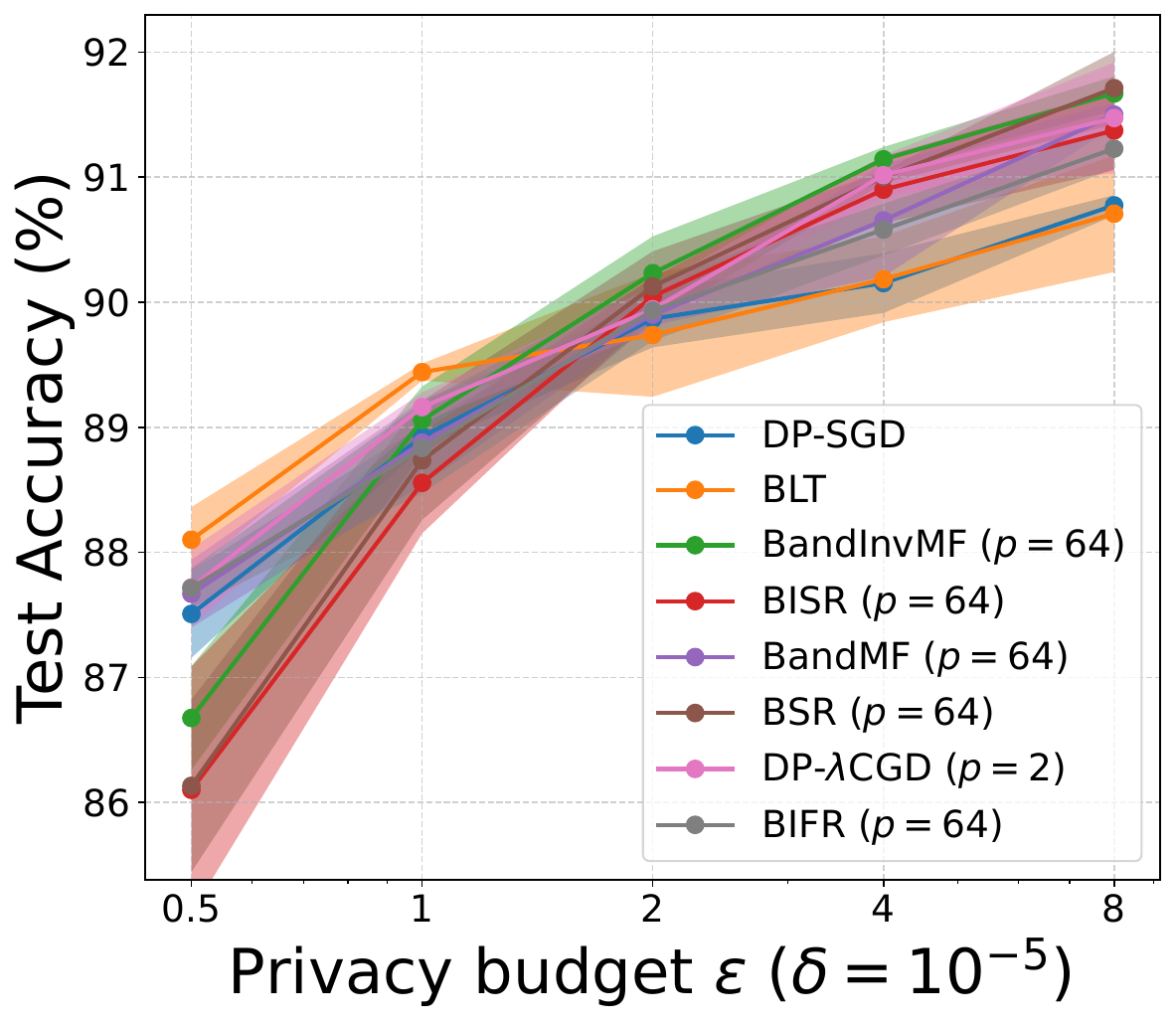}
        \caption{IMDb/BERT-Base. high memory}
        \label{fig:bifr_fixedp_highmem_imdb_bert}
    \end{subfigure}
    \caption{Test accuracy of BERT-Base LoRA (rank $4$) fine-tuning on IMDb dataset in the low- ($p=4$) and high-memory ($p=64$) regimes for different factorizations across privacy budgets $\varepsilon$.
    To reduce the variance in the training we additionally used momentum parameter $0.95$ to smooth the trajectory. The resulting accuracies (\Cref{tab:bifr-accuracy-ci-lowmem-imdb-bert,tab:bifr-accuracy-ci-highmem-imdb-bert}) are not sensitive enough to the method to draw decisive conclusions.
    }
    \label{fig:bifr_fixedp_imdb_bert}
\end{figure}

\begin{table}[h]
    \centering
    \small
    \caption{\method $\gamma$ values used in the fixed-$p$ experiments.
    The $p=2$ rows correspond to the DP-$\lambda$CGD comparison. These values are chosen based on the amplified RMSE.}
    \label{tab:bifr_exp_gamma}
    \begin{tabular}{llrrrrr}
        \toprule
        Dataset & $p$ & $\varepsilon=0.5$ & $\varepsilon=1$ &
        $\varepsilon=2$ & $\varepsilon=4$ & $\varepsilon=8$ \\
        \midrule
        CIFAR-10/100 & 2  & 0.80 & 0.90 & 0.90 & 0.95 & 0.95 \\
        CIFAR-10/100 & 4  & 0.60 & 0.70 & 0.80 & 0.85 & 0.85 \\
        CIFAR-10/100 & 64 & 0.30 & 0.30 & 0.35 & 0.40 & 0.45 \\
        IMDb & 2  & 0.70 & 0.80 & 0.85 & 0.90 & 0.90 \\
        IMDb & 4  & 0.50 & 0.55 & 0.70 & 0.75 & 0.80 \\
        IMDb & 64 & 0.15 & 0.20 & 0.30 & 0.35 & 0.40 \\
        \bottomrule
    \end{tabular}
\end{table}

\paragraph{Hyperparameter selection}
For Bayesian optimization, we separate $10\%$ of the training split for validation.
This gives $45{,}000$ training and $5{,}000$ validation examples for CIFAR-10/100, and $22{,}500$ training and $2{,}500$ validation examples for IMDb.
For each method, dataset, privacy budget, and seed, we tune only the learning rate with Optuna~\citep{akiba2019optuna}, using 25 trials over the log-uniform range $[10^{-3},0.3]$.
After selecting the learning rate, we retrain on the full training set and evaluate the model on the test set.
We report test accuracy for all tasks.

We fix the batch size $512$, clipping norm $10$, bandwidth for all methods, and also the $\gamma$ for \method.
We select the optimal \method $\gamma$ from a pre-computed (using $200$k MC samples) amplified-RMSE grid.
For CIFAR-10, these are the optimal values shown in \Cref{fig:bifr_optimal_gamma}.  For fine-tuning BERT-Base on IMDb, we additionally use a momentum parameter of $0.95$ to smooth the training trajectory.
\Cref{tab:bifr_exp_gamma} states the exact \method $\gamma$-values used in the experiments. 

We report the accuracies in \Cref{tab:bifr-accuracy-ci-lowmem-imdb-bert,tab:bifr-accuracy-ci-highmem-imdb-bert}. For the $\gamma$ selection, we also compared tuning $\gamma$ for the momentum workload, thereby accounting for the fact that momentum is used. However, we did not observe any significant improvement, and thus kept the same $\gamma$ as in the no-momentum setting in the plots.

\paragraph{Compute resources}
We ran each training run as one Slurm job on one full GPU node.
Each node consists of four AMD MI250X GPU modules, one 64-core AMD CPU, 512 GiB CPU memory, and 512 GiB total GPU memory.
Individual CIFAR-10/VGG runs took around 1--2 node-hours, and individual IMDb/BERT-Base runs took 4.5--6 node hours including hyperparameter selection.
The reported 480 runs used approximately 5k GPU hours.
Additionally, for development, preliminary sweeps, failed jobs, and debugging, we estimate that we used an extra 5--10k GPU hours.

\subsection{Private training experiment values}
\label{app:train_exp_exact_values}

In this appendix, we report the final test accuracy values used in \Cref{fig:bifr_fixedp_lowmem,fig:bifr_fixedp_imdb_bert,fig:bifr_fixedp_cifar100}.
The CIs are $95\%$ $t$-intervals computed over three repeats.

\begin{table}[h!]
    \centering
    \scriptsize
    \caption{Final test accuracy for CIFAR-100/VGG, low memory.
    Entries are mean accuracy in percent with $95\%$ confidence intervals over three repeats.}
    \label{tab:bifr-accuracy-cifar100-vgg-lowmem}
    \resizebox{\linewidth}{!}{%
    \begin{tabular}{lccccc}
        \toprule
        Method & $\varepsilon=0.5$ & $\varepsilon=1$ & $\varepsilon=2$ & $\varepsilon=4$ & $\varepsilon=8$ \\
        \midrule
        DP-SGD & $12.04 \pm 1.04$ & $16.31 \pm 1.60$ & $\mathbf{19.81 \pm 1.26}$ & $22.21 \pm 1.43$ & $24.70 \pm 2.27$ \\
        BLT & $\mathbf{13.28 \pm 2.17}$ & $\mathbf{17.19 \pm 1.25}$ & $19.65 \pm 1.72$ & $22.02 \pm 1.72$ & $24.78 \pm 1.62$ \\
        BandInvMF ($p=4$) & $12.02 \pm 0.84$ & $16.16 \pm 1.90$ & $19.37 \pm 1.97$ & $22.14 \pm 2.31$ & $25.26 \pm 2.21$ \\
        BISR ($p=4$) & $12.06 \pm 1.11$ & $15.81 \pm 1.50$ & $18.26 \pm 1.45$ & $20.56 \pm 1.73$ & $23.16 \pm 2.12$ \\
        BandMF ($p=4$) & $11.69 \pm 0.94$ & $14.61 \pm 1.80$ & $16.27 \pm 1.16$ & $18.36 \pm 2.06$ & $21.11 \pm 2.49$ \\
        BSR ($p=4$) & $11.81 \pm 1.16$ & $14.61 \pm 1.74$ & $16.26 \pm 1.33$ & $18.09 \pm 1.67$ & $20.66 \pm 1.76$ \\
        DP-$\lambda$CGD ($p=2$) & $12.10 \pm 0.94$ & $15.40 \pm 1.70$ & $19.03 \pm 1.06$ & $21.81 \pm 0.52$ & $26.40 \pm 1.53$ \\
        \method ($p=4$) & $12.01 \pm 1.02$ & $15.98 \pm 1.65$ & $19.17 \pm 1.15$ & $\mathbf{22.67 \pm 0.81}$ & $\mathbf{27.02 \pm 1.73}$ \\
        \bottomrule
    \end{tabular}%
    }
\end{table}

\begin{table}[h!]
    \centering
    \scriptsize
    \caption{Final test accuracy for CIFAR-100/VGG, high memory.
    Entries are mean accuracy in percent with $95\%$ confidence intervals over three repeats.}
    \label{tab:bifr-accuracy-cifar100-vgg-highmem}
    \resizebox{\linewidth}{!}{%
    \begin{tabular}{lccccc}
        \toprule
        Method & $\varepsilon=0.5$ & $\varepsilon=1$ & $\varepsilon=2$ & $\varepsilon=4$ & $\varepsilon=8$ \\
        \midrule
        DP-SGD & $12.04 \pm 1.04$ & $16.31 \pm 1.60$ & $19.81 \pm 1.26$ & $22.21 \pm 1.43$ & $24.70 \pm 2.27$ \\
        BLT & $\mathbf{13.28 \pm 2.17}$ & $\mathbf{17.19 \pm 1.25}$ & $19.65 \pm 1.72$ & $22.02 \pm 1.72$ & $24.78 \pm 1.62$ \\
        BandInvMF ($p=64$) & $8.85 \pm 1.44$ & $12.79 \pm 1.24$ & $17.26 \pm 1.72$ & $22.19 \pm 0.67$ & $28.30 \pm 1.48$ \\
        BISR ($p=64$) & $9.16 \pm 1.36$ & $13.39 \pm 1.33$ & $18.12 \pm 1.51$ & $23.33 \pm 1.31$ & $29.48 \pm 1.94$ \\
        BandMF ($p=64$) & $9.27 \pm 0.85$ & $13.37 \pm 1.19$ & $17.91 \pm 1.41$ & $23.19 \pm 1.11$ & $28.91 \pm 1.67$ \\
        BSR ($p=64$) & $10.79 \pm 0.58$ & $15.27 \pm 1.27$ & $19.90 \pm 0.94$ & $24.90 \pm 2.04$ & $30.35 \pm 2.42$ \\
        DP-$\lambda$CGD ($p=2$) & $12.10 \pm 0.94$ & $15.40 \pm 1.70$ & $19.03 \pm 1.06$ & $21.81 \pm 0.52$ & $26.40 \pm 1.53$ \\
        \method ($p=64$) & $11.61 \pm 0.99$ & $16.02 \pm 1.17$ & $\mathbf{20.18 \pm 1.24}$ & $\mathbf{24.99 \pm 1.16}$ & $\mathbf{30.61 \pm 2.35}$ \\
        \bottomrule
    \end{tabular}%
    }
\end{table}

\begin{table}[h!]
    \centering
    \scriptsize
    \caption{Final test accuracy for CIFAR-10/VGG, low memory.
    Entries are mean accuracy in percent with $95\%$ confidence intervals over three repeats.}
    \label{tab:bifr-accuracy-ci-lowmem-cifar10-vgg}
    \resizebox{\linewidth}{!}{%
    \begin{tabular}{lccccc}
        \toprule
        Method & $\varepsilon=0.5$ & $\varepsilon=1$ & $\varepsilon=2$ & $\varepsilon=4$ & $\varepsilon=8$ \\
        \midrule
        DP-SGD & $49.34 \pm 1.83$ & $53.56 \pm 1.69$ & $58.03 \pm 0.55$ & $\mathbf{61.86 \pm 2.80}$ & $64.52 \pm 3.09$ \\
        BLT & $\mathbf{51.09 \pm 1.97}$ & $\mathbf{55.49 \pm 2.82}$ & $\mathbf{58.12 \pm 2.89}$ & $61.03 \pm 3.78$ & $64.64 \pm 2.59$ \\
        BandInvMF ($p=4$) & $49.18 \pm 1.17$ & $54.10 \pm 1.58$ & $57.44 \pm 3.52$ & $60.75 \pm 3.48$ & $64.86 \pm 3.02$ \\
        BISR ($p=4$) & $49.38 \pm 1.35$ & $53.65 \pm 0.95$ & $56.32 \pm 2.66$ & $59.14 \pm 3.45$ & $62.51 \pm 3.09$ \\
        BandMF ($p=4$) & $48.58 \pm 1.85$ & $51.67 \pm 2.49$ & $54.01 \pm 2.45$ & $56.52 \pm 3.37$ & $59.64 \pm 3.71$ \\
        BSR ($p=4$) & $48.88 \pm 1.98$ & $52.20 \pm 1.71$ & $54.20 \pm 1.23$ & $56.35 \pm 3.01$ & $59.27 \pm 2.90$ \\
        DP-$\lambda$CGD ($p=2$) & $49.05 \pm 1.64$ & $52.74 \pm 0.86$ & $56.85 \pm 5.39$ & $59.81 \pm 1.94$ & $65.05 \pm 2.24$ \\
        \method ($p=4$) & $49.07 \pm 1.26$ & $53.70 \pm 2.10$ & $57.01 \pm 4.87$ & $61.38 \pm 3.94$ & $\mathbf{65.75 \pm 3.08}$ \\
        \bottomrule
    \end{tabular}%
    }
\end{table}

\begin{table}[h!]
    \centering
    \scriptsize
    \caption{Final test accuracy for CIFAR-10/VGG, high memory.
    Entries are mean accuracy in percent with $95\%$ confidence intervals over three repeats.}
    \label{tab:bifr-accuracy-ci-highmem-cifar10-vgg}
    \resizebox{\linewidth}{!}{%
    \begin{tabular}{lccccc}
        \toprule
        Method & $\varepsilon=0.5$ & $\varepsilon=1$ & $\varepsilon=2$ & $\varepsilon=4$ & $\varepsilon=8$ \\
        \midrule
        DP-SGD & $49.34 \pm 1.83$ & $53.56 \pm 1.69$ & $58.03 \pm 0.55$ & $61.86 \pm 2.80$ & $64.52 \pm 3.09$ \\
        BLT & $\mathbf{51.09 \pm 1.97}$ & $\mathbf{55.49 \pm 2.82}$ & $58.12 \pm 2.89$ & $61.03 \pm 3.78$ & $64.64 \pm 2.59$ \\
        BandInvMF ($p=64$) & $44.40 \pm 1.70$ & $49.10 \pm 1.44$ & $54.65 \pm 3.03$ & $60.93 \pm 2.14$ & $66.75 \pm 2.35$ \\
        BISR ($p=64$) & $45.53 \pm 1.35$ & $50.53 \pm 1.55$ & $55.86 \pm 1.84$ & $62.38 \pm 2.97$ & $67.91 \pm 2.61$ \\
        BandMF ($p=64$) & $45.37 \pm 1.90$ & $50.36 \pm 0.75$ & $55.75 \pm 2.18$ & $62.00 \pm 2.82$ & $67.16 \pm 2.53$ \\
        BSR ($p=64$) & $47.02 \pm 2.11$ & $52.73 \pm 1.33$ & $58.05 \pm 2.88$ & $64.25 \pm 3.60$ & $68.18 \pm 2.50$ \\
        DP-$\lambda$CGD ($p=2$) & $49.05 \pm 1.64$ & $52.74 \pm 0.86$ & $56.85 \pm 5.39$ & $59.81 \pm 1.94$ & $65.05 \pm 2.24$ \\
        \method ($p=64$) & $48.58 \pm 2.11$ & $54.14 \pm 1.84$ & $\mathbf{58.75 \pm 3.08}$ & $\mathbf{64.47 \pm 1.99}$ & $\mathbf{68.51 \pm 2.21}$ \\
        \bottomrule
    \end{tabular}%
    }
\end{table}

\begin{table}[h!]
    \centering
    \scriptsize
    \caption{Final test accuracy for IMDb/BERT LoRA rank 4, momentum 0.95, low memory.
    For $\gamma$-BIFR, we report results for both $\gamma$ settings: one tuned for the momentum workload and one tuned for the prefix-sum workload.
    Entries are mean accuracy in percent with $95\%$ confidence intervals over 3 repeats.}
    \label{tab:bifr-accuracy-ci-lowmem-imdb-bert}
    \resizebox{\linewidth}{!}{%
    \begin{tabular}{lccccc}
        \toprule
        Method & $\varepsilon=0.5$ & $\varepsilon=1$ & $\varepsilon=2$ & $\varepsilon=4$ & $\varepsilon=8$ \\
        \midrule
        DP-SGD
            & $87.51 \pm 1.07$
            & $88.93 \pm 0.26$
            & $89.87 \pm 0.70$
            & $90.15 \pm 0.73$
            & $90.77 \pm 0.25$ \\

        BLT
            & $\mathbf{88.10 \pm 0.80}$
            & $\mathbf{89.44 \pm 0.21}$
            & $89.74 \pm 1.51$
            & $90.18 \pm 1.04$
            & $90.71 \pm 1.42$ \\

        BandInvMF ($p=4$)
            & $87.02 \pm 1.00$
            & $89.15 \pm 0.82$
            & $89.98 \pm 0.54$
            & $90.85 \pm 0.43$
            & $91.57 \pm 1.48$ \\

        BISR ($p=4$)
            & $87.15 \pm 1.37$
            & $89.01 \pm 0.99$
            & $90.04 \pm 0.41$
            & $90.99 \pm 0.38$
            & $91.37 \pm 2.19$ \\

        BandMF ($p=4$)
            & $87.67 \pm 0.57$
            & $88.87 \pm 0.27$
            & $89.50 \pm 0.17$
            & $90.01 \pm 0.18$
            & $90.42 \pm 0.49$ \\

        BSR ($p=4$)
            & $87.70 \pm 0.49$
            & $88.86 \pm 0.42$
            & $89.61 \pm 0.28$
            & $89.81 \pm 1.13$
            & $90.45 \pm 0.68$ \\

        DP-$\lambda$CGD ($p=2$)
            & $87.72 \pm 0.94$
            & $89.16 \pm 0.35$
            & $89.94 \pm 0.53$
            & $\mathbf{91.02 \pm 0.48}$
            & $91.47 \pm 1.35$ \\

        \method ($p=4$)
            & $87.66 \pm 1.10$
            & $89.01 \pm 1.00$
            & $\mathbf{90.23 \pm 0.42}$
            & $90.45 \pm 1.45$
            & $91.00 \pm 2.21$ \\

        \method ($p=4$, w/ moment.)
            & $87.70 \pm 0.33$
            & $88.84 \pm 0.82$
            & $89.77 \pm 0.61$
            & $90.58 \pm 0.37$
            & $\mathbf{91.62 \pm 0.30}$ \\

        \bottomrule
    \end{tabular}%
    }
\end{table}
\begin{table}[h!]
    \centering
    \scriptsize
    \caption{Final test accuracy for IMDb/BERT LoRA rank 4, momentum 0.95, high memory.
    For $\gamma$-BIFR, we report results for both $\gamma$ settings: one tuned for the momentum workload and one tuned for the prefix-sum workload.
    Entries are mean accuracy in percent with $95\%$ confidence intervals over 3 repeats.}
    \label{tab:bifr-accuracy-ci-highmem-imdb-bert}
    \resizebox{\linewidth}{!}{%
    \begin{tabular}{lccccc}
        \toprule
        Method & $\varepsilon=0.5$ & $\varepsilon=1$ & $\varepsilon=2$ & $\varepsilon=4$ & $\varepsilon=8$ \\
        \midrule
        DP-SGD
            & $87.51 \pm 1.07$
            & $88.93 \pm 0.26$
            & $89.87 \pm 0.70$
            & $90.15 \pm 0.73$
            & $90.77 \pm 0.25$ \\

        BLT
            & $\mathbf{88.10 \pm 0.80}$
            & $\mathbf{89.44 \pm 0.21}$
            & $89.74 \pm 1.51$
            & $90.18 \pm 1.04$
            & $90.71 \pm 1.42$ \\

        BandInvMF ($p=64$)
            & $86.68 \pm 1.26$
            & $89.06 \pm 0.82$
            & $\mathbf{90.23 \pm 0.90}$
            & $\mathbf{91.14 \pm 0.30}$
            & $91.67 \pm 0.42$ \\

        BISR ($p=64$)
            & $86.10 \pm 3.00$
            & $88.56 \pm 1.23$
            & $90.05 \pm 0.46$
            & $90.90 \pm 0.81$
            & $91.37 \pm 0.97$ \\

        BandMF ($p=64$)
            & $87.67 \pm 0.82$
            & $88.88 \pm 1.03$
            & $89.90 \pm 0.26$
            & $90.65 \pm 1.37$
            & $91.50 \pm 0.29$ \\

        BSR ($p=64$)
            & $86.13 \pm 2.10$
            & $88.74 \pm 1.45$
            & $90.13 \pm 0.85$
            & $91.00 \pm 0.15$
            & $\mathbf{91.71 \pm 0.88}$ \\

        DP-$\lambda$CGD ($p=2$)
            & $87.72 \pm 0.94$
            & $89.16 \pm 0.35$
            & $89.94 \pm 0.53$
            & $91.02 \pm 0.48$
            & $91.47 \pm 1.35$ \\

        \method ($p=64$)
            & $87.63 \pm 0.66$
            & $88.87 \pm 0.73$
            & $90.03 \pm 0.10$
            & $90.56 \pm 1.16$
            & $91.25 \pm 0.99$ \\

        \method ($p=64$, w/ moment.)
            & $87.71 \pm 0.48$
            & $88.83 \pm 1.10$
            & $89.93 \pm 0.74$
            & $90.58 \pm 0.63$
            & $91.23 \pm 0.49$ \\

        \bottomrule
    \end{tabular}%
    }
\end{table}

\clearpage

\section{Amplified RMSE Analysis}
\label{sec:amplified_rmse}

\begin{table}[h!]
    \caption{Comparison of amplified RMSE under the \textbf{Balls-in-Bins accountant}, with non-amplified
RMSE shown in the last column, for different factorization methods with
$n = 2048$ and $k = 8$. For each value, we optimize the bandwidth $p$ over powers
of $2$, as well as $\gamma$, $c$, or $\nu$, where applicable. To get a tight
error estimate, we use $2 \times 10^6$ samples for the Balls-in-Bins accountant.
The proposed $\gamma$-banded and banded-inverse fraction-root factorizations
outperform the other factorizations in the high-privacy regime, including
BandMF w/ Opt., which is optimized for amplified RMSE. In the low-privacy regime,
BandMF w/ Opt. yields the lowest error, while BandMF achieves the best
non-amplified RMSE.}
    \centering

\begin{tabular}{lrrrrc}
\toprule
Method & $\varepsilon{=}1$ & $\varepsilon{=}2$ & $\varepsilon{=}4$ & $\varepsilon{=}8$ & $\varepsilon{=}8$ (Non-Amp)\\
\midrule
BISR & 22.99 & 13.23 & 7.60 & 4.74 & 6.75\\
$\gamma$-BIFR (ours) & \textbf{22.41} & 12.65  & 7.57 & 4.74 & 6.69 \\
BandInvMF   & 28.79 & 15.43 & 8.42 & 4.81 & 6.55 \\
DP-$\lambda$CGD & 25.15 & 16.35 & 10.90& 7.24& 9.68\\
\midrule
BandMF      & 25.35 & 14.59 & 8.42 & 4.83 & \textbf{6.35}\\
BandMF w/ Opt. &  22.80 & \textbf{12.60}  & \textbf{7.44} & \textbf{4.53} &   - \\

BandMF + Heuristic\quad \qquad & 24.71  & 13.34  & 7.57 & 4.82 &   - \\
BSR         &  23.24  &  13.16 & 7.82 & 4.55 & 6.57\\
$\gamma$-BFR (ours) & 22.65& \textbf{12.62} & \textbf{7.44} & \textbf{4.55} & 6.38 \\
$1/j + c$ & 22.54 & 12.71 & 7.50 & 4.62 & 6.54\\
\midrule
BLT & 29.26 & 15.50 & 8.48 & 4.87 & 6.65 \\
DP-$\nu$FTRL & 22.72 & 12.83 & 7.61& 4.82& 6.90\\
k-ary Tree & 48.18 & 25.88 & 14.16 & 8.15 & 12.53\\
\bottomrule
\end{tabular}    
    \label{tab:amplified_rmse}
\end{table}

\begin{table*}[h!]
\centering
\caption{
Comparison of amplified RMSE under \textbf{Poisson and $b$-min-sep
accountants} for different banded factorization methods, with $n=2048$.
For each setting, we optimize the bandwidth $p$ over powers of $2$, as
well as $\gamma$ or $c$, where applicable. Poisson accounting is exact,
whereas $b$-min-sep accounting requires Monte Carlo sampling; for
computational reasons, we use $2\times10^{5}$ samples.Each entry reports the mean $\pm$ standard deviation over five runs.
The lowest mean RMSE for each $(\varepsilon,k)$ setting is shown in bold.
}

\renewcommand{\arraystretch}{1.12}
\setlength{\tabcolsep}{4pt}

\resizebox{\textwidth}{!}{%
\begin{tabular}{lcccccc}
\toprule
& & \multicolumn{5}{c}{Number of participations ($k$)} \\
\cmidrule(lr){3-7}
Method & $\varepsilon$ & 1 & 2 & 4 & 8 & 16 \\
\midrule

\multirow{5}{*}{$\gamma$-BFR}
& 0.5
& $6.17 \pm 0.12$
& $11.33 \pm 0.22$
& $\mathbf{21.10 \pm 0.42}$
& $41.34 \pm 0.78$
& $\mathbf{80.94 \pm 1.30}$ \\
& 1.0
& $\mathbf{3.63 \pm 0.08}$
& $6.58 \pm 0.07$
& $11.58 \pm 0.27$
& $\mathbf{22.09 \pm 0.36}$
& $\mathbf{43.57 \pm 0.97}$ \\
& 2.0
& $\mathbf{2.39 \pm 0.04}$
& $\mathbf{3.90 \pm 0.04}$
& $\mathbf{6.65 \pm 0.15}$
& $\mathbf{12.15 \pm 0.25}$
& $\mathbf{23.81 \pm 0.42}$ \\
& 4.0
& $\mathbf{1.66 \pm 0.01}$
& $\mathbf{2.42 \pm 0.06}$
& $\mathbf{3.95 \pm 0.11}$
& $\mathbf{7.56 \pm 0.10}$
& $\mathbf{13.52 \pm 0.24}$ \\
& 8.0
& $\mathbf{1.21 \pm 0.02}$
& $\mathbf{1.65 \pm 0.04}$
& $\mathbf{2.48 \pm 0.04}$
& $\mathbf{4.57 \pm 0.07}$
& $\mathbf{7.45 \pm 0.13}$ \\

\addlinespace[2pt]
\midrule
\addlinespace[2pt]

\multirow{5}{*}{BSR}
& 0.5
& $6.88 \pm 0.14$
& $11.78 \pm 0.26$
& $21.76 \pm 0.38$
& $41.33 \pm 0.61$
& $81.37 \pm 0.89$ \\
& 1.0
& $4.49 \pm 0.09$
& $7.16 \pm 0.13$
& $12.72 \pm 0.25$
& $22.78 \pm 0.56$
& $44.20 \pm 0.92$ \\
& 2.0
& $3.12 \pm 0.07$
& $4.44 \pm 0.10$
& $7.24 \pm 0.17$
& $12.95 \pm 0.26$
& $24.07 \pm 0.42$ \\
& 4.0
& $1.89 \pm 0.03$
& $2.78 \pm 0.04$
& $4.32 \pm 0.12$
& $7.89 \pm 0.06$
& $13.55 \pm 0.25$ \\
& 8.0
& $1.24 \pm 0.02$
& $1.75 \pm 0.04$
& $2.60 \pm 0.05$
& $\mathbf{4.57 \pm 0.07}$
& $7.58 \pm 0.17$ \\

\addlinespace[2pt]
\midrule
\addlinespace[2pt]

\multirow{5}{*}{BandMF}
& 0.5
& $7.47 \pm 0.17$
& $12.50 \pm 0.24$
& $22.45 \pm 0.38$
& $42.38 \pm 0.87$
& $81.89 \pm 1.20$ \\
& 1.0
& $5.06 \pm 0.13$
& $7.76 \pm 0.17$
& $13.75 \pm 0.19$
& $23.63 \pm 0.52$
& $44.72 \pm 0.75$ \\
& 2.0
& $2.88 \pm 0.06$
& $4.93 \pm 0.09$
& $8.03 \pm 0.19$
& $13.73 \pm 0.28$
& $25.64 \pm 0.53$ \\
& 4.0
& $1.78 \pm 0.02$
& $2.85 \pm 0.05$
& $4.77 \pm 0.14$
& $8.23 \pm 0.15$
& $14.39 \pm 0.25$ \\
& 8.0
& $\mathbf{1.22 \pm 0.01}$
& $1.84 \pm 0.04$
& $2.86 \pm 0.04$
& $4.85 \pm 0.06$
& $8.47 \pm 0.13$ \\

\addlinespace[2pt]
\midrule
\addlinespace[2pt]

\multirow{5}{*}{$1/j+c$}
& 0.5
& $\mathbf{5.89 \pm 0.11}$
& $\mathbf{11.25 \pm 0.10}$
& $\mathbf{21.12 \pm 0.38}$
& $\mathbf{41.19 \pm 0.77}$
& $81.21 \pm 1.14$ \\
& 1.0
& $4.11 \pm 0.11$
& $\mathbf{6.29 \pm 0.15}$
& $\mathbf{11.54 \pm 0.22}$
& $22.18 \pm 0.27$
& $44.61 \pm 0.50$ \\
& 2.0
& $2.61 \pm 0.05$
& $\mathbf{3.89 \pm 0.07}$
& $\mathbf{6.66 \pm 0.15}$
& $12.21 \pm 0.25$
& $23.93 \pm 0.42$ \\
& 4.0
& $1.92 \pm 0.04$
& $2.56 \pm 0.04$
& $\mathbf{3.98 \pm 0.12}$
& $7.59 \pm 0.12$
& $13.53 \pm 0.25$ \\
& 8.0
& $1.39 \pm 0.03$
& $1.71 \pm 0.03$
& $\mathbf{2.49 \pm 0.05}$
& $4.68 \pm 0.08$
& $7.48 \pm 0.13$ \\

\addlinespace[2pt]
\midrule
\addlinespace[2pt]

\multirow{5}{*}{BandMF Pois.}
& 0.5
& $8.58$
& $13.49$
& $23.31$
& $43.06$
& $82.63$ \\
& 1.0
& $6.11$
& $8.79$
& $14.01$
& $24.45$
& $45.44$ \\
& 2.0
& $4.51$
& $6.07$
& $8.98$
& $14.61$
& $25.81$ \\
& 4.0
& $3.32$
& $4.29$
& $6.01$
& $9.19$
& $15.33$ \\
& 8.0
& $1.97$
& $2.81$
& $4.06$
& $5.95$
& $9.48$ \\

\bottomrule

\end{tabular}%
}
\label{tab:final_rmse}
\end{table*}

We compare different factorizations in terms of amplified RMSE. For amplification, we primarily use the Balls-in-Bins scheme, which is the most efficient and general scheme~\citep{choquette2024near}; see Table~\ref{tab:amplified_rmse}. For banded methods, in this section, we also consider cyclic Poisson amplification~\citep{choquette2023amplified}, a strict generalization of amplified DP-SGD, and the new $b$-min-sep sampling scheme~\citep{dong2026privacy}; see Table~\ref{tab:final_rmse}.

We divide the factorizations in Table~\ref{tab:amplified_rmse} into three groups. The first group, banded inverse methods, consists of Banded Inverse Square Root (BISR), the proposed \method, Banded Inverse Matrix Factorization (BandInvMF) \citep{kalinin2025back}, and DP-$\lambda$CGD \citep{kalinin2026dplambdacgd}. The second group, banded methods, includes BandMF \citep{scalingmckenna2024}; its version optimized for Balls-in-Bins accounting (BandMF w/ Opt.) \citep{choquette2024near}; its version with a heuristic (see Appendix H of \citet{scalingmckenna2024}); Banded Square Root (BSR) \citep{kalinin2024}; \(\gamma\)-BFR, a banded version of the proposed factorization; and \(1/j + c\), a single-parameter factorization with \(c > 0\), proposed in Appendix J of \citep{scalingmckenna2024}. The third group consists of non-banded factorizations, including BLT \citep{hasslefree2024} and DP-\(\nu\)FTRL \citep{choquette2023correlated}; for the latter, we optimize over the choice of \(\nu\) to achieve the best amplified RMSE.

We also compare against a \(k\)-ary tree mechanism, which predates matrix-factorization mechanisms and generalizes the binary tree mechanism to an arbitrary branching factor \citep{qardaji_2013, cardoso2022differentially, andersson2024count}.
The main challenge is that, when written naively as an encoding matrix, it is not lower triangular.
One can make the encoding matrix lower triangular by a QL decomposition without affecting the error, since the column inner products are preserved, as suggested in \citet{denisov2022improved}.
However, the resulting entries may become negative, preventing tight accounting in multi-participation settings, both with and without amplification.
Fortunately, standard $k$-ary tree mechanisms based on unweighted aggregation of nodes can be expressed as lower-triangular factorizations with entries in $\{0,1\}$.
By contrast, tree-based mechanisms based on more advanced aggregation techniques (such as leveraging \emph{subtraction} of nodes~\citep{honaker2015efficient, andersson2024count}) achieve lower RMSE in the single-participation setting, but are not known to admit factorizations with a positive and lower-triangular encoding matrix.
As a result, they are beyond the scope of current accounting methods.

Many of the methods in Table~\ref{tab:amplified_rmse} have never been considered in the multi-participation setting with amplification, making this the first thorough comparison in the amplified setting.

We also compare alternative amplification schemes for banded methods in Table~\ref{tab:final_rmse}, including cyclic Poisson amplification for BandMF (BandMF w/ Pois.). We note that the amplification effect depends only on the bandwidth and not on the coefficients, making BandMF automatically well suited to this amplification scheme. We also compare against the novel $b$-min-sep amplification scheme~\citep{dong2026privacy}, which leads to nontrivial differences depending on the coefficients. We evaluate it for BandMF (BandMF w/ $b$-min-sep), BSR, $\gamma$-BFR, a banded version of our proposed factorization, and $1/j + c$. We observe that the parametrized factorizations achieve the lowest amplified RMSE, significantly lower than the corresponding values under Balls-in-Bins subsampling.

\section{Proofs}

This section is structured as follows.
First, we prove the intermediate result of \Cref{lem:sens-toeplitz-bound}, which relates the sensitivity $\mathrm{sens}_{k,b}(\cdot)$ to the standard operator norms $\|\cdot\|_{1\to 1}$ and $\|\cdot\|_{1\to 2}$.
This allows for bounding the error $\mathcal{E}(B_\gamma^{(p)}, C_\gamma^{(p)})$ purely in terms of $\| B_{\gamma}^{(p)} \|_F$, $\|C_{\gamma}^{(p)} \|_{1\to 1}$ and $\| C_{\gamma}^{(p)}\|_{1\to 2}$.
We proceed to compute tight bounds on the coefficients of $B_{\gamma}^{(p)}$ and $C_{\gamma}^{(p)}$, from which we derive bounds on the corresponding matrix norms, and consequently the error (\Cref{cor:rmse}).
Lastly, we compute the error for different choices of $\gamma \in [1/2, 1)$, which ultimately yields \Cref{thm:bifr-multi-participation}.

Throughout $O(\cdot)$ only hides absolute constants independent of $n, p, k, b$ and $\gamma$.
We also note that $\log(\cdot)$ always refers to the natural logarithm with base $e$.
We will also use the convention that any constant divided by zero is $\infty$, especially in the context of taking minimum of two expressions.

\subsection{Bounding the Sensitivity via Operator Norms}

Before proving \Cref{lem:sens-toeplitz-bound}, we briefly recount the operator norms $\|\cdot \|_{1\to 1}$ and $\| \cdot \|_{1\to 2}$ that feature in it.
For any real-valued matrix $M$, $\| M \|_{1 \to 1}$ and $\|M\|_{1\to 2}$ correspond to the largest $\ell_1$ respectively $\ell_2$ norm of any column in $M$.
In the particular case of lower-triangular Toeplitz matrices, the largest such norms are always achieved for the first column.

\begin{proof}[Proof of \Cref{lem:sens-toeplitz-bound}]
Throughout the proof we zero-index the columns of $C$, and for $j\in\{0, \dots, n-1\}$ let $C_j := C_{[\cdot,\, j]}$ denote the $j$-th column of $C$.
By \Cref{thm:b-sensitivity},
\begin{align}
    \mathrm{sens}_{k,b}(C)^2
    = \left\|\sum_{j=0}^{k-1}C_{jb}\right\|_2^2
    = \sum_{i=0}^{k-1}\sum_{j=0}^{k-1}\langle C_{ib},\,C_{jb}\rangle
    \leq k \langle C_0, C_0 \rangle + 2\sum_{0 \leq i < j \leq k-1} \langle C_{ib}, C_{jb} \rangle.
\end{align}
where the inequality uses that $\langle C_j, C_j \rangle = \sum_{r=0}^{n-1-j} c_r^2 \leq \sum_{r=0}^{n-1} c_r^2 = \langle C_0, C_0 \rangle$.
The remainder of the proof deals with the cross-term sum.

Note that for any of the terms $\langle C_{ib}, C_{jb} \rangle$ (where $i < j$), we have that
\begin{align}
    \langle C_{ib}, C_{jb} \rangle
    &= \sum_{\ell = jb}^{n-1} c_{\ell - ib} c_{\ell - jb}
    = \sum_{r=0}^{n - 1 - jb}  c_{r + (j-i) b} c_{r}\\
    &\leq \sum_{r=0}^{n - 1 - (j-i)b}  c_{r + (j-i) b} c_{r}
    = \sum_{\ell = (j-i)b}^{n-1} c_{\ell} c_{\ell - (j - i)b}
    = \langle C_0, C_{(j-i)b} \rangle,
\end{align}
where the inequality uses that expanding the summation range of non-negative summands can only increase the value of the sum.
Hence, we can upper-bound the cross-term sum as
\begin{equation}\label{eq:cross-term-sens}
    2\sum_{0 \leq i < j \leq k-1} \langle C_{ib}, C_{jb} \rangle
    \leq 2 \sum_{m=1}^{k-1} (k-m) \langle C_0, C_{mb} \rangle
    \leq 2k \sum_{m=1}^{k-1} \langle C_0, C_{mb} \rangle.
\end{equation}

To upper bound $\langle C_0, C_{mb}\rangle$, we use an averaging argument.
First, note that the map $q \mapsto \langle C_0, C_{q} \rangle = \sum_{r=0}^{n-1-q} c_{r} c_{r+q}$ is \emph{decreasing} in $q$, as $0 \leq c_{r+q+1} \leq c_{r+q}$ (each summand decreases in value), and the summation range decreases with $q$.
In particular, this implies that for each $m \geq 1$, $\langle C_0, C_{mb}\rangle$ is the minimum taken over the window $h\in[(m-1)b + 1, mb]$, and so is bounded by the average:
\begin{equation}\label{eq:average-trick-sens}
   \langle C_0, C_{mb} \rangle \leq \frac{1}{b} \sum_{h=(m-1)b + 1}^{mb} \langle C_0, C_{h} \rangle.
\end{equation}
Continuing from \Cref{eq:cross-term-sens} and applying \Cref{eq:average-trick-sens}:
\begin{align}
    2k \sum_{m=1}^{k-1} \langle C_0, C_{mb} \rangle
    \leq \frac{2k}{b} \sum_{m=1}^{k-1} \sum_{h=(m-1)b + 1}^{mb} \langle C_0, C_{h} \rangle
    = \frac{2k}{b} \sum_{h=1}^{(k-1)b} \langle C_0, C_{h} \rangle
    \leq \frac{2k}{b} \sum_{h=1}^{n-1} \langle C_0, C_{h} \rangle
\end{align}
where the first equality uses that the windows $\{ [(m-1)b + 1, mb] : 1\leq m\leq k-1\}$ exactly tile $[1, (k-1)b]$, and the last inequality that $(k-1)b \leq n-1$.

To finish the proof, we will argue that $\sum_{h=1}^{n-1} \langle C_0, C_h \rangle = \frac{1}{2}(\|C_0\|_1^2 - \| C_0 \|_2^2)$.
Indeed,
\begin{align}
    \sum_{h=0}^{n-1} \langle C_0, C_h \rangle
    = \sum_{h=0}^{n-1}\sum_{r=0}^{n-1-h} c_{r} c_{r+h}
    = \sum_{0 \leq r \leq r' \leq n - 1} c_{r} c_{r'}
    = \frac{1}{2}\left(\|C_0\|_{1}^2 + \| C_0\|_{2}^2\right),
\end{align}
where the last step uses that every term $c_r c_{r'}$ appears twice in $\| C_0 \|_1^2$, except for $c_r^2$ which only appears once, which we correct by adding $\| C_0 \|_2^2 = \sum_r c_r^2$, after which dividing by $2$ yields the equality.
Hence,
\begin{align}
   \sum_{h=1}^{n-1} \langle C_0, C_h \rangle
   = -\langle C_0, C_0\rangle + \sum_{h=0}^{n-1} \langle C_0, C_h \rangle
   = \frac{1}{2}(\|C_0\|_1^2 - \| C_0 \|_2^2).
\end{align}
Putting everything together,
\begin{align}
    \mathrm{sens}_{k, b}(C)^2
    &\leq k \| C_0 \|_{2}^2 + \frac{k}{b} \left( \|C_0\|_{1}^2 - \| C_0 \|_{2}^2\right)\\
    &= k\left(1-\frac{1}{b}\right) \| C_0 \|_{2}^2 + \frac{k}{b}\|C_0\|_{1}^2\\
    &\leq k \| C_0 \|_{2}^2 + \frac{k}{b}\|C_0\|_{1}^2,
\end{align}
where the final statement follows from identifying $\| C_0 \|_{2}^2 = \| C \|_{1\to 2}^2$ and $\|C_0\|_{1}^2 = \|C\|_{1\to 1}^2$, as each of the norms correspond to largest column norm measured in either $\ell_2$ or $\ell_1$.

\paragraph{Lower Bound.}
First, we rewrite the matrix inequality $\mathrm{sens}_{k,b}(C)^2 \ge \frac{k}{2}\|C\|^2_{1\to 2} + \frac{k}{3b}\|C\|^2_{1\to 1}$ in terms of sums of the subdiagonal elements $c_j$. In particular, let us split $[0,n - 1]$ into $k$ segments of the form $\ell + jb$ for $\ell \in [0, b-1]$ and $j \in [0, k -1]$:
\begin{equation}
    \mathrm{sens}_{k,b}(C)^2 = \sum\limits_{\ell = 0}^{b - 1} \sum\limits_{j = 0}^{k - 1} \left(\sum\limits_{i =0}^{j} c_{\ell + ib}\right)^2 \ge \frac{k}{2}\sum\limits_{\ell = 0}^{b - 1}\sum\limits_{j = 0}^{k - 1}c_{\ell + jb}^2 + \frac{k}{3b} \left(\sum\limits_{\ell = 0}^{b - 1} \sum\limits_{j = 0}^{k - 1}c_{\ell + jb}\right)^2.
\end{equation}
It is then sufficient to prove the following inequality for any fixed value of $\ell$:
\begin{equation}
\label{ineq:sens_k_b_lower_bound_fixed_ell}
    \sum\limits_{j = 0}^{k - 1} \left(\sum\limits_{i =0}^{j} c_{\ell + ib}\right)^2 \ge \frac{k}{2}\sum\limits_{j = 0}^{k - 1}c_{\ell + jb}^2 + \frac{k}{3} \left(\sum\limits_{j = 0}^{k - 1}c_{\ell + jb}\right)^2.
\end{equation}
Indeed, the first two sums are obtained simply by summing over $\ell$; for the third sum, we need to use Cauchy--Schwarz:
\begin{equation}
    \frac{k}{3} \sum\limits_{\ell = 0}^{b - 1}\left(\sum\limits_{j = 0}^{k - 1}c_{\ell + jb}\right)^2 \ge\frac{k}{3b} \left(\sum\limits_{\ell = 0}^{b - 1} \sum\limits_{j = 0}^{k - 1}c_{\ell + jb}\right)^2.
\end{equation}
Denoting $x_\ell = \sum\limits_{j = 0}^{k - 1}c_{\ell + jb}$, this turns into the simple inequality $b\sum_{\ell =0}^{b - 1}x_{\ell}^2 \ge (\sum_{\ell = 0}^{b - 1}x_{\ell})^2$. It remains to prove Inequality~\ref{ineq:sens_k_b_lower_bound_fixed_ell}. For this, we will use the monotonicity of $c_{\ell + ib}$. Let us denote $d_{i} =c_{\ell + ib} - c_{\ell + (i + 1)b} \ge 0$; now we can rewrite the inequality in terms of $d_i$, with $d_{k - 1} = c_{\ell + (k - 1)b}$. In particular,
\begin{equation}
\begin{aligned}
     \sum\limits_{j = 0}^{k - 1} \left(\sum\limits_{i =0}^{j} c_{\ell + ib}\right)^2 &= \sum\limits_{j = 0}^{k - 1} \left(\sum\limits_{i =0}^{j} \sum\limits_{r = i}^{k - 1}d_r\right)^2 = \sum\limits_{j = 0}^{k - 1} \left(\sum\limits_{i =0}^{k - 1} d_i \min(j + 1, i + 1)\right)^2\\
     & =\sum\limits_{i = 0}^{k - 1}\sum\limits_{r = 0}^{k - 1}d_id_r \sum\limits_{j = 0}^{k - 1}\min(j + 1, i + 1)\min(j + 1, r + 1).
\end{aligned}
\end{equation}
The second sum becomes
\begin{equation}
    \frac{k}{2}\sum\limits_{j = 0}^{k - 1}c_{\ell + jb}^2 = \frac{k}{2}\sum\limits_{j = 0}^{k - 1} \left(\sum\limits_{i =j}^{k - 1} d_i\right)^2 =  \frac{k}{2}\sum\limits_{j = 0}^{k - 1} \sum\limits_{i = j}^{k - 1}\sum\limits_{r = j}^{k - 1}d_i d_r = \frac{k}{2} \sum\limits_{i = 0}^{k - 1}\sum\limits_{r = 0}^{k - 1}d_i d_r \min(i + 1, r + 1).
\end{equation}
And the third sum gives
\begin{equation}
    \frac{k}{3} \left(\sum\limits_{j = 0}^{k - 1}c_{\ell + jb}\right)^2 = \frac{k}{3} \left(\sum\limits_{j = 0}^{k - 1}\sum\limits_{i = j}^{k - 1}d_i\right)^2 = \frac{k}{3} \left(\sum\limits_{i = 0}^{k - 1}d_i (i + 1)\right)^2 = \frac{k}{3}\sum\limits_{i = 0}^{k - 1} \sum\limits_{r = 0}^{k - 1} d_i d_r (i + 1)(r + 1).
\end{equation}
Then it suffices to prove, for all $i, r \in[0, k -1]$, that
\begin{equation}
\label{ineq:i_r_sums_inequality}
 \sum\limits_{j = 0}^{k - 1}\min(j + 1, i + 1)\min(j + 1, r + 1) \ge \frac{k}{2}\min(i + 1, r + 1) + \frac{k}{3}(i + 1)(r + 1).
\end{equation}
Then, multiplying it by the non-negative quantity $d_i d_r$ and summing over $i,r$ proves Inequality~\ref{ineq:sens_k_b_lower_bound_fixed_ell}. Note that Inequality~\ref{ineq:i_r_sums_inequality} is symmetric in the variables $i,r$. Thus, without loss of generality, we may assume that $i \le r$, giving
\begin{equation}
\begin{aligned}
    \sum\limits_{j = 0}^{k - 1}\min(j + 1, i + 1)\min(j + 1, r + 1) &= \sum\limits_{j = 0}^{i} (j + 1)^2 + \sum\limits_{j = i + 1}^{r} (j + 1)(i + 1)\\
    &\qquad+ \sum\limits_{j = r + 1}^{k - 1} (i + 1)(r+1)\\
    &= \frac{(i + 1)(i + 2)(2i + 3)}{6}  +\frac{(i + 1)(r - i ) (i + r + 3)}{2}\\
    &\qquad + (i + 1)(r + 1)(k - 1 - r)\\
    &= (i + 1)\left[k(r + 1) - \frac{r(r + 1)}{2} - \frac{i ( i + 2)}{6}\right].
\end{aligned}
\end{equation}
Then, dividing both sides of Inequality~\ref{ineq:i_r_sums_inequality} by the positive quantity $i + 1$, we obtain the following equivalent inequality:
\begin{equation}
    k(r + 1) - \frac{r(r + 1)}{2} - \frac{i ( i + 2)}{6} \ge \frac{k}{2} + \frac{k}{3}(r + 1).
\end{equation}
Subtracting the right-hand side, it suffices to show that
\begin{equation}
    \frac{k}{6} + \frac{2kr}{3} -\frac{r( r + 1)}{2} - \frac{i (i + 2)}{6} \ge 0.
\end{equation}
Now we use the assumption that $i \le r$:

\begin{equation}
    \frac{k}{6} + \frac{2kr}{3} -\frac{r( r + 1)}{2} - \frac{i (i + 2)}{6} \ge \frac{k(1+4r)}{6} - \frac{r(4r + 5)}{6} = \frac{(k - 1 - r)(1+4r) + 1}{6} \ge 0,
\end{equation}

where in the last inequality we used that $0\le r \le k - 1$, concluding the proof.
\end{proof}

\subsection{Coefficient Bounds}

In this section we give bounds on coefficient sequences closely related to the \method factorization.
In particular, we recall from \Cref{def:gamma_bifr} that $C_{\gamma}^{(p)}$ is defined through its inverse where
\begin{equation*}  
\bigl(C_\gamma^{(p)}\bigr)^{-1} := 
\begin{pmatrix} 
1 & 0 & \cdots & 0 & \cdots &0\\ 
(\tilde c_{\gamma})_1 & 1 & \cdots & 0 & \cdots &0 \\ 
\vdots & \vdots & \ddots & \vdots & \ddots & \vdots\\ 
(\tilde c_{\gamma})_{p - 1} & (\tilde c_{\gamma})_{p-2} & \cdots & 1 & \cdots &0\\
0 & (\tilde c_{\gamma})_{p - 1} & \cdots & (\tilde c_{\gamma})_1 &\cdots & 0 \\ 
\vdots & \vdots & \ddots & \vdots & \ddots & \vdots \\ 
0 & 0 & \cdots & (\tilde c_{\gamma})_{p - 1}&\cdots & 1 
\end{pmatrix}, 
\qquad \text{for} \quad (\tilde c_{\gamma})_i := (-1)^i \binom{\gamma}{i}, 
\end{equation*}
and $\binom{\gamma}{i} = \prod_{j = 1}^{i} \frac{\gamma + 1 - j}{j}$ is a generalized binomial coefficient.
For the corresponding subdiagonals of the Toeplitz matrix $C_{\gamma}^{(p)}$ we use the notation $((c_{\gamma}^{(p)})_j)_{j\geq 0}$, which we bound in \Cref{lem:calpha-p-geometric}.
Additionally, observe that $B_{\gamma}^{(p)} C_{\gamma}^{(p)} = E$, and so $B_{\gamma}^{(p)}$ equals $E (C_{\gamma}^{(p)})^{-1}$.
Since $E$ is the prefix-sum matrix, the corresponding entries of $B_{\gamma}^{(p)}$ are prefix sums on $((\tilde{c}_\gamma)_j)_{j\geq 0}$, and we bound these in \Cref{cor:prefix-sums-ctilde}.

In what follows \Cref{lem:bounds-calpha}~and~\ref{lem:bounds-ctilde} produce bounds on the sequences $((c_\gamma)_j)_{j\geq 0}$ and $((\tilde{c}_\gamma)_j)_{j\geq 0}$; \Cref{cor:prefix-sums-ctilde} bounds prefix sums on $((\tilde{c}_\gamma)_j)_{j\geq 0}$ and \Cref{lem:calpha-p-monotone}~and~\ref{lem:calpha-p-geometric} bound $((c_{\gamma}^{(p)})_j)_{j\geq 0}$.

Before we get going, however, we first need two brief helper lemmas.
\begin{lemma}\label{lem:harm-sum-bound}
    Let $s > 0$, $m \geq 1$ an integer and define $H_{m}^{(s)} := \sum_{r=1}^{m} r^{-s}$.
    Then
    \begin{equation*}
        H_m^{(s)} \le 1+\max\{1,m^{1-s}\}\cdot\min\Bigl\{\log m,\ \frac{1}{|1-s|}\Bigr\}.
    \end{equation*}
    Moreover, if $s < 1$, then $H_m^{(s)} \leq m^{1-s}/(1-s)$.
\end{lemma}
\begin{proof}
    Since the summands are non-increasing,
    \begin{equation*}
        H_m^{(s)} \le 1 + I, \qquad I := \int_1^m x^{-s}\,dx = \int_0^{\log m} e^{(1-s)t}\,dt .
    \end{equation*}
    The integrand is at most $\max\{1,m^{1-s}\}$, so $I \le \max\{1,m^{1-s}\}\log m$.
    For $s\neq1$, $I=(m^{1-s}-1)/(1-s)$. If $s>1$, this gives $I=(1-m^{1-s})/(s-1)\le 1/(s-1)$.
    If $s<1$, it gives $I\le m^{1-s}/(1-s)$. This proves the first bound.
    For the second statement, if $s<1$ then
    \begin{equation*}
        1+I=\frac{m^{1-s}-s}{1-s}\le\frac{m^{1-s}}{1-s}. \qedhere
    \end{equation*}
\end{proof}
\begin{lemma}[Gamma facts]\label{lem:gamma-facts}
    Let $z> 0$ and $\Gamma$ denote the Gamma function.
    The following facts hold:
    \begin{enumerate}
        \item $\Gamma(z+1) = z\Gamma(z)$.
        \item If $z\in(0, 2)$, then $\Gamma(z) \geq 1/2$.
        \item If $z\in(1, 2)$, then $\Gamma(z) \leq 1$.
    \end{enumerate}
\end{lemma}

\begin{lemma}[Bounds for $(c_\gamma)_j$]
\label{lem:bounds-calpha}
Let $0<\gamma<1$ and define $c_\gamma=\bigl((c_\gamma)_j\bigr)_{j\ge 0}$ by
\begin{equation}
\label{eq:calpha-def}
(1-x)^{-\gamma}=\sum_{j=0}^{\infty} (c_\gamma)_j\,x^j,\qquad |x|<1.
\end{equation}
Then, for every $j\ge 0$,
\begin{equation}
\label{eq:calpha-gamma-form}
(c_\gamma)_j=\binom{\gamma+j-1}{j}
=\frac{\Gamma(\gamma+j)}{\Gamma(\gamma)\,\Gamma(j+1)}.
\end{equation}
Moreover, for every integer $j\ge 1$,
\begin{equation}
\label{eq:calpha-bounds}
\frac{1}{\Gamma(\gamma)\,(j+1)^{1-\gamma}}
\;<\;
(c_\gamma)_j
\;<\;
\frac{1}{\Gamma(\gamma)\,j^{1-\gamma}}.
\end{equation}
\end{lemma}
\begin{proof}
The binomial series gives
\begin{equation}
(c_\gamma)_j=(-1)^j\binom{-\gamma}{j}=\binom{\gamma+j-1}{j}=\frac{\Gamma(\gamma+j)}{\Gamma(\gamma)\,\Gamma(j+1)}.
\end{equation}
For the bounds, rewrite
\begin{equation}
(c_\gamma)_j
=\frac{1}{\Gamma(\gamma)}\,\frac{\Gamma(j+\gamma)}{\Gamma(j+1)}.
\end{equation}
By Gautschi's inequality (applied with $x=j$ and $s=\gamma\in(0,1)$),
\begin{equation}
(j+1)^{\gamma-1}
<\frac{\Gamma(j+\gamma)}{\Gamma(j+1)}
<j^{\gamma-1}.
\end{equation}
Multiplying by $1/\Gamma(\gamma)$ yields \Cref{eq:calpha-bounds}.
\end{proof}

\begin{lemma}[Bounds for $(\tilde c_\gamma)_j$]
\label{lem:bounds-ctilde}
Let $0<\gamma<1$ and write
\begin{equation}
(1-x)^\gamma=\sum_{j=0}^\infty (\tilde c_\gamma)_j\,x^j,\qquad |x|<1.
\end{equation}
Then $(\tilde c_\gamma)_0=1$ and, for every $j\ge 1$,
\begin{equation}
(\tilde c_\gamma)_j = (-1)^j\binom{\gamma}{j}
=-\,\frac{\gamma\,\Gamma(j-\gamma)}{\Gamma(1-\gamma)\,\Gamma(j+1)}.
\end{equation}
Moreover, for every $j\ge2$,
\begin{equation}
\frac{\gamma}{\Gamma(1-\gamma)}\,\frac{1}{j^{\gamma+1}}
\;<\;
|(\tilde c_\gamma)_j|
\;<\;
\frac{\gamma}{\Gamma(1-\gamma)}\,\frac{1}{j\,(j-1)^{\gamma}}.
\end{equation}
\end{lemma}

\begin{proof}
The binomial series for $(1-x)^\gamma$ gives
\begin{equation}
(1-x)^\gamma=\sum_{j=0}^{\infty}(-1)^j\binom{\gamma}{j}x^j,\qquad |x|<1,
\end{equation}
so $(\tilde c_\gamma)_0=1$ and $(\tilde c_\gamma)_j=(-1)^j\binom{\gamma}{j}$ for $j\ge 1$.

For $j\ge 1$, start from the binomial-coefficient definition
\begin{equation}
\binom{\gamma}{j}=\frac{\gamma(\gamma-1)\cdots(\gamma-j+1)}{j!}.
\end{equation}
Factoring out $-\gamma$ and rewriting the remaining product via Gamma functions gives
\begin{equation}
(-1)^j\binom{\gamma}{j}
=-\,\frac{\gamma(1-\gamma)(2-\gamma)\cdots(j-1-\gamma)}{j!}
=-\,\frac{\gamma}{\Gamma(j+1)}\,\frac{\Gamma(j-\gamma)}{\Gamma(1-\gamma)}
=-\,\frac{\gamma\,\Gamma(j-\gamma)}{\Gamma(1-\gamma)\,\Gamma(j+1)}.
\end{equation}
This yields the claimed closed form for $(\tilde c_\gamma)_j$.

For $j\ge 2$, rewrite
\begin{equation}
|(\tilde c_\gamma)_j|
=\frac{\gamma}{\Gamma(1-\gamma)}\,\frac{\Gamma(j-\gamma)}{\Gamma(j+1)}
=\frac{\gamma}{\Gamma(1-\gamma)}\,\frac{1}{j}\,\frac{\Gamma(j-\gamma)}{\Gamma(j)}.
\end{equation}
Apply Gautschi's inequality to $\Gamma(j-\gamma)/\Gamma(j)$ with $x=j-1$ and $s=1-\gamma\in(0,1)$:
\begin{equation}
\frac{\gamma}{\Gamma(1-\gamma)}\,\frac{1}{j^{\gamma+1}}
\;<\;
|(\tilde c_\gamma)_j|
\;<\;
\frac{\gamma}{\Gamma(1-\gamma)}\,\frac{1}{j\,(j-1)^{\gamma}},
\end{equation}
which is exactly the stated two-sided bound.
\end{proof}

\begin{corollary}[Prefix sums of $\tilde c_\gamma$]
\label{cor:prefix-sums-ctilde}
Let $0<\gamma<1$ and let $\tilde c_\gamma=\bigl((\tilde c_\gamma)_j\bigr)_{j\ge 0}$ be defined by
$(1-x)^\gamma=\sum_{j=0}^\infty (\tilde c_\gamma)_j x^j$ for $|x|<1$.
Define the prefix sums
\begin{equation}
s_{\gamma,j}:=\sum_{i=0}^{j}(\tilde c_\gamma)_i\qquad (j\ge 0).
\end{equation}
Then $s_{\gamma,0}=1$ and for every $j\ge 0$,
\begin{equation}
\label{eq:prefix-sum-ctilde-is-c}
s_{\gamma,j}=(c_{1-\gamma})_j.
\end{equation}
In particular, for every integer $j\ge 1$,
\begin{equation}
\label{eq:prefix-sum-ctilde-bounds}
\frac{1}{\Gamma(1-\gamma)\,(j+1)^{\gamma}}
\;<\;
\sum_{i=0}^{j}(\tilde c_\gamma)_i
\;<\;
\frac{1}{\Gamma(1-\gamma)\,j^{\gamma}}.
\end{equation}
\end{corollary}

\begin{proof}
Let $S_\gamma(x):=\sum_{j=0}^{\infty} s_{\gamma,j}x^j$ be the generating function of the prefix sums.
By definition of $s_{\gamma,j}$,
\begin{equation}
S_\gamma(x)
=\sum_{j=0}^{\infty}\left(\sum_{i=0}^{j}(\tilde c_\gamma)_i\right)x^j
=\left(\sum_{i=0}^{\infty}(\tilde c_\gamma)_i x^i\right)\left(\sum_{k=0}^{\infty}x^k\right),
\end{equation}
so for $|x|<1$,
\begin{equation}
S_\gamma(x)
=(1-x)^\gamma\cdot (1-x)^{-1}=(1-x)^{-1+\gamma}=(1-x)^{-(1-\gamma)}.
\end{equation}
Thus $s_{\gamma,j}$ is exactly the $j$-th coefficient in the expansion of $(1-x)^{-(1-\gamma)}$, i.e.\
$s_{\gamma,j}=(c_{1-\gamma})_j$, proving \Cref{eq:prefix-sum-ctilde-is-c}.
The bounds \Cref{eq:prefix-sum-ctilde-bounds} follow from \Cref{lem:bounds-calpha} applied with $\gamma$ replaced by $1-\gamma$.
\end{proof}

Bounding $((c_{\gamma}^{(p)})_j)_{j\geq 0}$ is more complicated.
As a first step, we prove that the entries are monotonically decreasing.
\begin{lemma}[Monotonicity of $c_{\gamma}^{(p)}$]
\label{lem:calpha-p-monotone}
Let $0<\gamma<1$ and let the coefficients $c_{\gamma}^{(p)}=\bigl((c_{\gamma}^{(p)})_j\bigr)_{j\ge 0}$ denote the diagonal values of $C^{(p)}_{\gamma}$ from \Cref{def:gamma_bifr}.
Then
\begin{equation}
    0 \le (c_{\gamma}^{(p)})_{j+1}\le (c_{\gamma}^{(p)})_{j}\qquad (j\ge 0).
\end{equation}
\end{lemma}
\begin{proof}
The coefficients satisfy the recursion
\begin{equation}\label{eq:recursion-calpha-p}
(c_{\gamma}^{(p)})_j
= -\sum_{k=1}^{\min\{p-1,j\}} (\tilde c_\gamma)_k \,(c_{\gamma}^{(p)})_{j-k}.
\end{equation}
Since $(\tilde c_\gamma)_k\le 0$ for $k\ge1$, the right-hand side is a nonnegative linear combination of
earlier coefficients. With $(c_{\gamma}^{(p)})_0=1$, it follows by induction that
$(c_{\gamma}^{(p)})_j\ge0$ for all $j\ge0$.

For monotonicity, define
\begin{equation}
\Delta_j := (c_{\gamma}^{(p)})_{j}-(c_{\gamma}^{(p)})_{j+1}\qquad (j\ge 0).
\end{equation}
If $p=1$, then the sum in \Cref{eq:recursion-calpha-p} is empty and $(c_{\gamma}^{(p)})_j=0$ for $j\ge1$,
so monotonicity is immediate. Assume $p\ge2$. Using \Cref{eq:recursion-calpha-p} with $j=1$ gives
$(c_{\gamma}^{(p)})_1=-(\tilde c_\gamma)_1(c_{\gamma}^{(p)})_0=\gamma$, hence $\Delta_0=1-\gamma>0$.

Now apply \Cref{eq:recursion-calpha-p} at indices $j+1$ and $j+2$ and subtract:
\begin{equation}
\Delta_{j+1}
=\sum_{k=1}^{\min\{p-1,j+1\}} \bigl(-(\tilde c_\gamma)_k\bigr)\,\Delta_{j+1-k}
\;+\;
\begin{cases}
\;(\tilde c_\gamma)_{j+2}, & j+2\le p-1,\\[2pt]
0, & j+2>p-1.
\end{cases}
\end{equation}
Since $-(\tilde c_\gamma)_k\ge0$ for all $k\ge1$, the right-hand side is a sum of nonnegative terms whenever
$\Delta_0,\ldots,\Delta_j\ge0$ \emph{and} $j+2 > p-1$.
When $j+2 \leq p-1$, the right-hand side is however not obviously positive.
For this case, we instead write
\begin{align}
   \Delta_{j+1} &= (c_\gamma^{(p)})_{j+1} - (c_\gamma^{(p)})_{j+2} 
   = (c_\gamma)_{j+1} - (c_\gamma)_{j+2} \\
   &> \frac{1}{\Gamma(\gamma) (j+2)^{1-\gamma}} - \frac{1}{\Gamma(\gamma) (j+2)^{1-\gamma}}
   = 0,
\end{align}
where we have used $(c_{\gamma}^{(p)})_\ell$ equals $(c_\gamma)_\ell$ for $\ell \leq p-1$, and the last step bounds on $(c_\gamma)_\ell$ are from \Cref{eq:calpha-bounds}.

Therefore, by induction, $\Delta_j\ge0$ for all $j\ge0$, i.e.,
$(c_{\gamma}^{(p)})_{j+1}\le (c_{\gamma}^{(p)})_{j}$.
\end{proof}

Our bound on $((c_{\gamma}^{(p)})_j)_{j\geq 0}$ is based on showing a geometric decay, with rate parametrized in $p$ and $\gamma$, for coefficients with index greater than the bandwidth $p$.
It is stated and proved next.
\begin{lemma}[Geometric tail bound for $c_\gamma^{(p)}$]
\label{lem:calpha-p-geometric}
Let $0<\gamma<1$ and $p\ge 2$, and define
\begin{equation}
\label{eq:beta-def}
\beta_{\gamma,p}
\;:=\;\min\!\left\{
\frac{1-\gamma}{2p}\log\Bigl(\frac{2p}{p+1}\Bigr),\;
\frac{1-\gamma}{8\,p^{\gamma}\,(p-1)^{1-\gamma}}
\right\}.
\end{equation}
Then, for every $j\ge p$,
\begin{equation}
\label{eq:c-decay-sharp}
(c_\gamma^{(p)})_j \leq (c_\gamma)_p\,\bigl(1-\beta_{\gamma,p}\bigr)^{j-p}.
\end{equation}
Moreover, $\beta_{\gamma,p} = \Theta( (1-\gamma)/p)$ and in particular satisfies the two-sided bound
\begin{equation}
\label{eq:beta-two-sided}
    \frac{1-\gamma}{8p} \leq \beta_{\gamma,p}\leq \frac{1-\gamma}{4p}.
\end{equation}
\end{lemma}
\begin{proof}
We first establish the two-sided bound \Cref{eq:beta-two-sided} on $\beta_{\gamma,p}$.
Denote the two terms in \Cref{eq:beta-def} by
\begin{equation}
    T_1 := \frac{1-\gamma}{2p}\log \Bigl(\frac{2p}{p+1}\Bigr), \qquad T_2 := \frac{1-\gamma}{8 p^{\gamma} (p-1)^{1-\gamma}}.
\end{equation}
For $T_1$: since $p\mapsto 2p/(p+1) = 2-2/(p+1)$ is increasing and equals $4/3$ at $p=2$, $\log(2p/(p+1))\geq\log(4/3) > 1/4$ for $p\ge 2$, hence $T_1 \ge (1-\gamma)/(8p)$.
For $T_2$: rewrite
\begin{equation}
    T_2 = \frac{1-\gamma}{8p (1-1/p)^{1-\gamma}}.
\end{equation}
Since the base $(1-1/p)\in(0,1)$ and the exponent $(1-\gamma)\in(0,1)$, we have the elementary bounds
\begin{equation}
    \frac{p-1}{p} = (1-1/p)^{1} \leq (1-1/p)^{1-\gamma} \leq 1.
\end{equation}
The upper bound on $(1-1/p)^{1-\gamma}$ gives $T_2 \geq (1-\gamma)/(8p)$, and the lower bound gives $T_2 \leq (1-\gamma)/(8(p-1)) \leq (1-\gamma)/(4p)$ for $p\ge 2$. Combining,
\begin{equation}
    \frac{1-\gamma}{8p} \leq \min\{T_1,T_2\}=\beta_{\gamma,p}
    \leq T_2 \leq \frac{1-\gamma}{4p},
\end{equation}
proving \Cref{eq:beta-two-sided}.
In particular, $\beta_{\gamma,p}\le 1/(4p)\le 1/8$.

We proceed to prove the main statement.
Recall the recursion from \Cref{eq:recursion-calpha-p}, restated next:
\begin{equation}
(c_{\gamma}^{(p)})_j
= -\sum_{k=1}^{\min\{p-1,j\}} (\tilde c_\gamma)_k \,(c_{\gamma}^{(p)})_{j-k}
\qquad (j\ge0).
\end{equation}
Since $(\tilde c_\gamma)_k\le0$ for $k\ge1$ and $(c_{\gamma}^{(p)})_0=1$, \Cref{eq:recursion-calpha-p}
implies $(c_{\gamma}^{(p)})_j\ge0$ for all $j$ by induction.

\medskip
\noindent\emph{Step 1.} $(c_{\gamma}^{(p)})_j\le (c_\gamma)_j$ for all $j\ge0$.
The exact coefficients satisfy, for $j\ge1$,
\begin{equation}
(c_\gamma)_j=-\sum_{k=1}^{j}(\tilde c_\gamma)_k\,(c_\gamma)_{j-k}.
\end{equation}
For $0\le j\le p-1$ the two recursions coincide, hence $(c_{\gamma}^{(p)})_j=(c_\gamma)_j$.
Fix $j\ge p$ and assume $(c_{\gamma}^{(p)})_m\le (c_\gamma)_m$ for all $m<j$. Then, using
$-(\tilde c_\gamma)_k\ge0$ and $(c_\gamma)_m\ge0$,
\begin{equation}
(c_{\gamma}^{(p)})_j
=\sum_{k=1}^{p-1}\bigl(-(\tilde c_\gamma)_k\bigr)(c_{\gamma}^{(p)})_{j-k}
\le \sum_{k=1}^{p-1}\bigl(-(\tilde c_\gamma)_k\bigr)(c_\gamma)_{j-k}
\le \sum_{k=1}^{j}\bigl(-(\tilde c_\gamma)_k\bigr)(c_\gamma)_{j-k}
=(c_\gamma)_j.
\end{equation}

\medskip
\noindent\emph{Step 2.} Base range $p\le j\le 2p$.
The ratio
\begin{equation}
\frac{(c_\gamma)_{j+1}}{(c_\gamma)_j}=\frac{j+\gamma}{j+1}
\end{equation}
is increasing in $j$, hence $(c_\gamma)_j$ is log-convex. Therefore, for $r=0,1,\dots,p$,
\begin{equation}\label{eq:logconvex-interp}
(c_\gamma)_{p+r}\le (c_\gamma)_p^{\,1-r/p}\,(c_\gamma)_{2p}^{\,r/p}.
\end{equation}
Thus it suffices to prove the anchor inequality
\begin{equation}\label{eq:anchor-2p}
(c_\gamma)_{2p}\le (c_\gamma)_p\,(1-\beta_{\gamma,p})^p,
\end{equation}
because inserting \Cref{eq:anchor-2p} into \Cref{eq:logconvex-interp} yields
\begin{equation}
(c_\gamma)_{p+r}
\le (c_\gamma)_p^{\,1-r/p}\Bigl((c_\gamma)_p(1-\beta_{\gamma,p})^p\Bigr)^{r/p}
= (c_\gamma)_p(1-\beta_{\gamma,p})^{r},
\end{equation}
which is exactly the desired bound for $p\le p+r\le 2p$.

To prove \Cref{eq:anchor-2p}, note from \Cref{lem:bounds-calpha} that
\[
(c_\gamma)_{2p}<\frac{1}{\Gamma(\gamma)\,(2p)^{1-\gamma}},
\qquad
(c_\gamma)_p>\frac{1}{\Gamma(\gamma)\,(p+1)^{1-\gamma}},
\]
hence
\begin{equation}\label{eq:ratio-calpha}
    \frac{(c_\gamma)_{2p}}{(c_\gamma)_p} < \left(\frac{p+1}{2p}\right)^{1-\gamma}.
\end{equation}
Next, since $\beta_{\gamma,p}\le \frac12$, the inequality $\log(1-x)\ge -2x$ for $x\in[0,\frac12]$
implies
\begin{equation}
(1-\beta_{\gamma,p})^p
=\exp\big(p\log(1-\beta_{\gamma,p})\big)
\ge \exp(-2p\beta_{\gamma,p}).
\end{equation}
Because $\beta_{\gamma,p}\le \frac{1-\gamma}{2p}\log \big(\frac{2p}{p+1}\big)$, we obtain
\begin{equation}
\exp(-2p\beta_{\gamma,p})
\geq \exp \left(-(1-\gamma)\log \Bigl(\frac{2p}{p+1}\Bigr)\right)
=\left(\frac{p+1}{2p}\right)^{1-\gamma}.
\end{equation}
Combining with \Cref{eq:ratio-calpha} yields
\begin{equation}
(1-\beta_{\gamma,p})^p \ge \frac{(c_\gamma)_{2p}}{(c_\gamma)_p},
\end{equation}
which is exactly \Cref{eq:anchor-2p}. Consequently, for $r=0,1,\dots,p$,
\begin{equation}
(c_\gamma)_{p+r}\le (c_\gamma)_p(1-\beta_{\gamma,p})^{r}.
\end{equation}
Together with Step~1, this gives for all $p\le j\le 2p$,
\begin{equation}
(c_{\gamma}^{(p)})_j \le (c_\gamma)_j \le (c_\gamma)_p(1-\beta_{\gamma,p})^{j-p}.
\end{equation}

\medskip
\noindent\emph{Step 3.} Induction for all $j\ge 2p+1$.
Assume that for some $j\ge2p+1$,
\begin{equation}
(c_{\gamma}^{(p)})_m \le (c_\gamma)_p(1-\beta_{\gamma,p})^{\,m-p}
\qquad \text{for all } m=p,p+1,\dots,j-1.
\end{equation}
Then $j-k\ge p$ for $k=1,\dots,p-1$, and \Cref{eq:recursion-calpha-p} gives
\begin{equation}
(c_{\gamma}^{(p)})_j
=\sum_{k=1}^{p-1}\bigl(-(\tilde c_\gamma)_k\bigr)(c_{\gamma}^{(p)})_{j-k}
\le (c_\gamma)_p(1-\beta_{\gamma,p})^{\,j-p}
\sum_{k=1}^{p-1}\bigl(-(\tilde c_\gamma)_k\bigr)(1-\beta_{\gamma,p})^{-k}.
\end{equation}
To finish the inductive step, and thereby the full proof, we will show that
\begin{equation}
    \sum_{k=1}^{p-1}\bigl(-(\tilde c_\gamma)_k\bigr)(1-\beta_{\gamma,p})^{-k}
    \leq 1.\label{eq:sharp-step3-finisher}
\end{equation}
To this end, write $(1-\beta_{\gamma, p})^{-k} = \exp( - k\log(1 - \beta_{\gamma, p}))$.
Firstly, $-\log(1-x) \leq 2x$ on $[0, \frac{1}{2}]$, and as $\beta_{\gamma,p} \leq \frac{1}{4p} \leq \frac{1}{2}$, we note that
\begin{equation}
   -k\log(1-\beta_{\gamma, p})
   \leq -(p-1)\log(1-\beta_{\gamma, p})
   \leq 2(p-1)\beta_{\gamma, p}
   \leq \frac{2(p-1)}{4p}
   \leq \frac{1}{2}.
\end{equation}
where $k\leq p-1$ was used at the first step.
Secondly, since for $x\in[0,1]$ we have the standard inequality $e^x \leq 1 + 2x$, we write
\begin{equation}
    \exp(-k\log(1-\beta_{\gamma, p}))
    \leq 1 + 2(-k\log(1-\beta_{\gamma,p}))
    \leq 1+4k\beta_{\gamma, p}.
\end{equation}
Substituting this bound into \Cref{eq:sharp-step3-finisher}:
\begin{equation}
    \sum_{k=1}^{p-1}\bigl(-(\tilde c_\gamma)_k\bigr)(1-\beta_{\gamma,p})^{-k}
    \leq \underbrace{\sum_{k=1}^{p-1}\bigl(-(\tilde c_\gamma)_k\bigr)}_{1 - (c_{1-\gamma})_{p-1}}
    + 4\beta_{\gamma, p}\sum_{k=1}^{p-1}k \bigl(-(\tilde c_\gamma)_k\bigr).\label{eq:geom-pf-step1}
\end{equation}
We proceed to bound the second sum using \Cref{lem:bounds-ctilde}: $-(\tilde{c}_{\gamma})_1 = \gamma$, and for $k\geq 2$,
\begin{equation}
    k(-\tilde{c}_\gamma)_k < \frac{\gamma}{\Gamma(1-\gamma)(k-1)^{\gamma}}.
\end{equation}
By \Cref{lem:harm-sum-bound},
\begin{equation}
\sum_{k=2}^{p-1} k (-(\tilde{c}_\gamma)_k) < \frac{\gamma}{\Gamma(1-\gamma)}\sum_{j=1}^{p-2} j^{-\gamma} \leq \frac{\gamma (p-1)^{1-\gamma}}{(1-\gamma)\Gamma(1-\gamma)},
\end{equation}
hence
\begin{equation}
    \sum_{k=1}^{p-1} k (-(\tilde{c}_{\gamma})_k)
    \leq \gamma + \frac{\gamma (p-1)^{1-\gamma}}{(1-\gamma)\Gamma(1-\gamma)}
    \leq \frac{2 \gamma (p-1)^{1-\gamma}}{(1-\gamma)\Gamma(1-\gamma)},\label{eq:geom-pf-step2}
\end{equation}
where the last step absorbs the $\gamma$ term: by \Cref{lem:gamma-facts}, $(1-\gamma)\Gamma(1-\gamma) = \Gamma(2-\gamma) \leq 1$ for $\gamma \in (0,1)$, so $\gamma \leq \gamma(p-1)^{1-\gamma}/((1-\gamma)\Gamma(1-\gamma))$ for $p\geq 2$.

Plugging \Cref{eq:geom-pf-step2} into \Cref{eq:geom-pf-step1}
\begin{equation}
    \sum_{k=1}^{p-1} (-(\tilde{c}_{\gamma})_k) (1-\beta_{\gamma, p})^{-k}
    \leq 1 - (c_{1-\gamma})_{p-1} + \frac{8\gamma\beta_{\gamma, p} (p-1)^{1-\gamma}}{(1-\gamma)\Gamma(1-\gamma)}.\label{eq:geom-pf-step3}
\end{equation}
By \Cref{lem:bounds-calpha}, $(c_{1-\gamma})_{p-1} \geq 1 / (\Gamma(1-\gamma) p^{\gamma})$.
Taking the second branch of the minimum in \Cref{eq:beta-def}, we get that
\begin{equation}
    \frac{8\beta_{\gamma, p} \gamma(p-1)^{1-\gamma}}{(1-\gamma)\Gamma(1-\gamma)}
    \leq \frac{\gamma}{\Gamma(1-\gamma)p^{\gamma}}
    \leq \frac{1}{\Gamma(1-\gamma)p^{\gamma}}
    \leq (c_{1-\gamma})_{p-1}.
\end{equation}
It follows that the right-hand side of \Cref{eq:geom-pf-step3} is at most 1, establishing \Cref{eq:sharp-step3-finisher} and hence the inductive step.
The lemma follows.
\end{proof}

\subsection{Matrix Norm Bounds}
In this section we derive matrix norm bounds for the factors in the $\gamma$-BIFR factorization $B_{\gamma}^{(p)} C_{\gamma}^{(p)} = E$.
We begin with a bound on the Frobenius norm of $B_{\gamma}^{(p)}$.
\begin{lemma}[Frobenius norm bound for $B_{\gamma}^{(p)}$]
\label{lem:frob-Bap-bound}
Let $0<\gamma<1$ and $n\ge p\ge2$, and let $B_{\gamma}^{(p)}$ be as in \Cref{def:gamma_bifr}. Then
\begin{equation}\label{eq:frob-bound-exact}
\|B_{\gamma}^{(p)}\|_F^2
\;\le\;
n
+\frac{n}{\Gamma(1-\gamma)^2}\,H_{p-1}^{(2\gamma)}
+\frac{(n-p)(n-p+1)}{2\,\Gamma(1-\gamma)^2}\cdot\frac{1}{(p-1)^{2\gamma}},
\end{equation}
where $H_{m}^{(s)}:=\sum_{r=1}^{m} r^{-s}$.
In particular, the following asymptotic bound holds for $\gamma \in [1/2, 1)$:
\begin{equation}\label{eq:frob-bound-asympt}
    \frac{1}{\sqrt{n}}\|B_{\gamma}^{(p)}\|_F = O\bigl(1+(1-\gamma)\sqrt{\Lambda_\gamma}+(1-\gamma)\sqrt n\,p^{-\gamma}\bigr),
\end{equation}
where $\Lambda_\gamma = \min\{ \log(ep), 1 / \lvert 2\gamma - 1 \rvert \} \geq 1$.
\end{lemma}

\begin{proof}
By construction, $B_{\gamma}^{(p)}=E(C_{\gamma}^{(p)})^{-1}$ is a lower-triangular Toeplitz matrix.
Writing $r=i-j\in\{0,1,\dots,n-1\}$, the Toeplitz coefficient on the $r$-th subdiagonal equals the partial sum
\begin{equation}
\sum_{k=0}^{\min\{p-1,r\}}(\tilde c_\gamma)_k,
\end{equation}
and this subdiagonal contains $n-r$ entries. Therefore,
\begin{equation}
\|B_{\gamma}^{(p)}\|_F^2
=\sum_{r=0}^{n-1}(n-r)\left(\sum_{k=0}^{\min\{p-1,r\}}(\tilde c_\gamma)_k\right)^2.
\end{equation}
For $r=0$ the inner sum equals $(\tilde c_\gamma)_0=1$, giving the contribution $n$.
For $1\le r\le p-1$, \Cref{cor:prefix-sums-ctilde} yields
\begin{equation}
\sum_{k=0}^{r}(\tilde c_\gamma)_k=(c_{1-\gamma})_r,
\end{equation}
and by \Cref{eq:prefix-sum-ctilde-bounds} we have
\begin{equation}
0<\sum_{k=0}^{r}(\tilde c_\gamma)_k< \frac{1}{\Gamma(1-\gamma)\,r^\gamma}.
\end{equation}
For $r\ge p$ the inner sum is constant and equals $\sum_{k=0}^{p-1}(\tilde c_\gamma)_k$, which is bounded by
$\bigl(\Gamma(1-\gamma)\,(p-1)^\gamma\bigr)^{-1}$ in the same way. Hence
\begin{equation}
\|B_{\gamma}^{(p)}\|_F^2
\le n
+\frac{1}{\Gamma(1-\gamma)^2}\sum_{r=1}^{p-1}\frac{n-r}{r^{2\gamma}}
+\frac{1}{\Gamma(1-\gamma)^2}\left(\sum_{r=p}^{n-1}(n-r)\right)\frac{1}{(p-1)^{2\gamma}}.
\end{equation}
Using $\sum_{r=p}^{n-1}(n-r)=\sum_{m=1}^{n-p}m=\frac{(n-p)(n-p+1)}{2}$ gives
\begin{equation}
\|B_{\gamma}^{(p)}\|_F^2
\le n
+\frac{1}{\Gamma(1-\gamma)^2}\sum_{r=1}^{p-1}\frac{n-r}{r^{2\gamma}}
+\frac{(n-p)(n-p+1)}{2\,\Gamma(1-\gamma)^2}\cdot\frac{1}{(p-1)^{2\gamma}}.
\end{equation}
Finally, since $n-r\le n$ for $1\le r\le p-1$,
\begin{equation}
\sum_{r=1}^{p-1}\frac{n-r}{r^{2\gamma}}
\le n\sum_{r=1}^{p-1}\frac{1}{r^{2\gamma}}
= n\,H_{p-1}^{(2\gamma)}.
\end{equation}
Substituting this bound completes the proof of \Cref{eq:frob-bound-exact}.

\noindent\emph{Asymptotic bounds.}
To derive the asymptotic bound, we first rewrite \Cref{eq:frob-bound-exact}.
By the definition of the Gamma function, $\Gamma(1-\gamma) = \Gamma(2-\gamma) / (1-\gamma)$, and that $\Gamma(z) \geq 1/2$ for $z\in(1, 2)$ (\Cref{lem:gamma-facts}), we have that $1/\Gamma(1-\gamma)^2 \leq 4(1-\gamma)^2$.
Secondly,
\begin{equation}
    \frac{(n-p)(n-p+1)}{(p-1)^{2\gamma}}
    \leq \frac{n^2}{(p/2)^{2\gamma}}
    \leq 4 n^2p^{-2\gamma}
\end{equation}
where the first step used $p-1 \geq p/2$ for $p\geq 2$ and the last step that $2^{2\gamma} \leq 2^2 = 4$.
Combining these observations, we arrive at the coarser bound:
\begin{equation}\label{eq:frob-bound-coarse}
       \| B_{\gamma}^{(p)}\|_F^2
       \leq n + 4(1-\gamma)^2 n H_{p-1}^{(2\gamma)} + 16(1-\gamma)^2 n^2 p^{-2\gamma}\,.
\end{equation}
Now, using \Cref{lem:harm-sum-bound} when $\gamma \in [1/2, 1)$, we get that
\begin{equation*}
    H_{p-1}^{(2\gamma)} \leq 1 + \min\{\log(p-1), 1 / (2\gamma - 1)\} \leq 1 + \Lambda_\gamma,
\end{equation*}
and so, bounding the right-hand side of \Cref{eq:frob-bound-coarse}:
\begin{align}
       \| B_{\gamma}^{(p)}\|_F^2
       &\leq n + 4(1-\gamma)^2(1 + \Lambda_\gamma) n + 16(1-\gamma)^2n^2 p^{-2\gamma}\\
       &= n\left(1 + 4(1-\gamma)^2(1 + \Lambda_\gamma) + 16 (1-\gamma)^2 n p^{-2\gamma} \right)\\
       &= O\left(n\left(1 + (1-\gamma)^2\Lambda_\gamma + (1-\gamma)^2 n p^{-2\gamma}\right)\right).
\end{align}
To get the final statement, take the square root for each bound, use that $\sqrt{a+b} \leq \sqrt{a} + \sqrt{b}$ for positive $a, b$, and divide by $\sqrt{n}$.
\end{proof}

We proceed to use our geometric tail bound on $((c_\gamma^{(p)})_j)_{j\geq 0}$ to derive a bound on $\| C_{\gamma}^{(p)}\|_{1\to 2}$.
\begin{lemma}[Bound on $1\to 2$ norm of $C_{\gamma}^{(p)}$]
\label{lem:l2-sens}
Let $0<\gamma<1$ and $n\ge p\ge2$, and let $C_{\gamma}^{(p)}$ be as in \Cref{def:gamma_bifr}.
Additionally let $\beta_{\gamma, p}$ be as in \Cref{lem:calpha-p-geometric}.
Then
\begin{equation}\label{eq:l2-sens}
    \|C_{\gamma}^{(p)}\|_{1\to2}^2
    \;\le\;
    1+\frac{1}{\Gamma(\gamma)^2}\,H_{p-1}^{(2-2\gamma)}
    +\frac{1}{\Gamma(\gamma)^2\,p^{2-2\gamma}}\cdot\frac{1}{\beta_{\gamma,p}\,(2-\beta_{\gamma,p})},
\end{equation}
where $H_m^{(s)}:=\sum_{j=1}^{m}j^{-s}$.
In particular, the following asymptotic bound holds for $\gamma \in [1/2, 1)$:
\begin{equation}\label{eq:l2-sens-asympt}
    \|C_{\gamma}^{(p)}\|_{1\to 2} = O\bigl(p^{\gamma-\frac12}(\sqrt{\Lambda_\gamma}+1/\sqrt{1-\gamma})\bigr),
\end{equation}
where $\Lambda_\gamma = \min\{ \log(ep), 1 / \lvert 2\gamma - 1\rvert \} \geq 1$.
\end{lemma}

\begin{proof}
Since $C_{\gamma}^{(p)}$ is lower-triangular Toeplitz and $(c_{\gamma}^{(p)})_j\ge0$, the column $\ell_2$-norms
are nonincreasing from left to right; hence
\begin{equation}
\|C_{\gamma}^{(p)}\|_{1\to2}^2=\max_{1\le j\le n}\|C_{\gamma}^{(p)}e_j\|_2^2=\|C_{\gamma}^{(p)}e_1\|_2^2
=\sum_{r=0}^{n-1}(c_{\gamma}^{(p)})_r^2.
\end{equation}
Split the sum at $p$:
\begin{equation}
\sum_{r=0}^{n-1}(c_{\gamma}^{(p)})_r^2
=\sum_{r=0}^{p-1}(c_{\gamma}^{(p)})_r^2+\sum_{r=p}^{n-1}(c_{\gamma}^{(p)})_r^2.
\end{equation}
For $0\le r\le p-1$ we have $(c_{\gamma}^{(p)})_r=(c_\gamma)_r$, and by \Cref{lem:bounds-calpha},
$(c_\gamma)_0=1$ and $(c_\gamma)_r\le(\Gamma(\gamma)r^{1-\gamma})^{-1}$ for $r\ge1$, hence
\begin{equation}
\sum_{r=0}^{p-1}(c_{\gamma}^{(p)})_r^2
\le 1+\frac{1}{\Gamma(\gamma)^2}\sum_{r=1}^{p-1}\frac{1}{r^{2-2\gamma}}
=1+\frac{1}{\Gamma(\gamma)^2}\,H_{p-1}^{(2-2\gamma)}.
\end{equation}
For $r\ge p$, \Cref{lem:calpha-p-geometric} gives
\begin{equation}
(c_{\gamma}^{(p)})_r\le (c_\gamma)_p(1-\beta_{\gamma,p})^{\,r-p},
\end{equation}
so using $(c_\gamma)_p\le (\Gamma(\gamma)p^{1-\gamma})^{-1}$ from \Cref{lem:bounds-calpha},
\begin{equation}
\sum_{r=p}^{n-1}(c_{\gamma}^{(p)})_r^2
\le (c_\gamma)_p^{\,2}\sum_{m=0}^{\infty}(1-\beta_{\gamma,p})^{2m}
=\frac{(c_\gamma)_p^{\,2}}{1-(1-\beta_{\gamma,p})^2}
\le \frac{1}{\Gamma(\gamma)^2\,p^{2-2\gamma}}\cdot\frac{1}{\beta_{\gamma,p}(2-\beta_{\gamma,p})}.
\end{equation}
Combining the two parts yields the bound in \Cref{eq:l2-sens}.

\noindent\emph{Asymptotic bound.}
For the asymptotic statement, note that for $\gamma\in(0, 1)$, $\Gamma(\gamma) \geq 1/2$.
Additionally using that $\frac{1-\gamma}{8p} \leq \beta_{\gamma, p} \leq \frac{1-\gamma}{4p} \leq 1$, we can arrive at a coarser version of \Cref{eq:l2-sens}
\begin{equation}\label{eq:l2-sens-coarse}
    \|C_{\gamma}^{(p)}\|_{1\to2}^2
    \leq 1 + 4 H_{p-1}^{(2-2\gamma)} + \frac{32 p^{2\gamma - 1}}{1-\gamma}.
\end{equation}
To arrive at \Cref{eq:l2-sens-asympt}, we lastly invoke \Cref{lem:harm-sum-bound} under the assumption of $\gamma\in[1/2, 1)$ for
\begin{equation}
    H_{p-1}^{(2-2\gamma)} \leq 1 + (p-1)^{2\gamma - 1}\min\{ \log(p-1), 1 / (2\gamma -1)\} \leq 1 + p^{2\gamma - 1} \Lambda_\gamma,
\end{equation}
hence, inserting the result into \Cref{eq:l2-sens-coarse} and taking a square-root,
\begin{equation}
    \|C_{\gamma}^{(p)}\|_{1\to2}
    \leq \sqrt{5 + 4p^{2\gamma - 1} \Lambda_\gamma + \frac{32 p^{2\gamma - 1}}{1-\gamma}}
    = O(p^{\gamma -1/2}(\sqrt{\Lambda_\gamma} + 1/\sqrt{1-\gamma})),
\end{equation}
where the additive constant can be absorbed as $p^{2\gamma-1} \geq 1$ and $1/(1-\gamma) \geq 1$.
This completes the proof.
\end{proof}

The bound on $\|C_{\gamma}^{(p)}\|_{1\to 1}$ is derived analogously to \Cref{lem:l2-sens}, and is given next.
\begin{lemma}[Bound on $1\to 1$ norm of $C_{\gamma}^{(p)}$]
\label{lem:l1-sens}
Let $0<\gamma<1$ and $n\ge p\ge2$, and let $C_{\gamma}^{(p)}$ be as in \Cref{def:gamma_bifr}.
Additionally let $\beta_{\gamma, p}$ be as in \Cref{lem:calpha-p-geometric}.
Then
\begin{equation}\label{eq:l1-sens}
    \|C_{\gamma}^{(p)}\|_{1\to1}
    \;\le\;
    1+\frac{1}{\Gamma(\gamma)}\,H_{p-1}^{(1-\gamma)}
    +\frac{1}{\Gamma(\gamma)\,p^{1-\gamma} \beta_{\gamma, p}},
\end{equation}
where $H_m^{(s)}:=\sum_{j=1}^{m}j^{-s}$.
In particular, the following asymptotic bound holds for $\gamma \in (0, 1)$:
\begin{equation}\label{eq:l1-sens-asympt}
    \|C_{\gamma}^{(p)}\|_{1\to 1} = O\!\left(\frac{p^{\gamma}}{\gamma(1-\gamma)}\right).
\end{equation}
\end{lemma}
\begin{proof}
Since $C_{\gamma}^{(p)}$ is lower-triangular Toeplitz and $(c_{\gamma}^{(p)})_j\ge0$, the column $\ell_1$-norms
are nonincreasing from left to right; hence
\begin{equation}
\|C_{\gamma}^{(p)}\|_{1\to1}=\max_{1\le j\le n}\|C_{\gamma}^{(p)}e_j\|_1=\|C_{\gamma}^{(p)}e_1\|_1
=\sum_{r=0}^{n-1}(c_{\gamma}^{(p)})_r.
\end{equation}
Split the sum at $p$:
\begin{equation}
\sum_{r=0}^{n-1}(c_{\gamma}^{(p)})_r
=\sum_{r=0}^{p-1}(c_{\gamma}^{(p)})_r + \sum_{r=p}^{n-1}(c_{\gamma}^{(p)})_r.
\end{equation}
For $0\le r\le p-1$ we have $(c_{\gamma}^{(p)})_r=(c_\gamma)_r$, and by \Cref{lem:bounds-calpha},
$(c_\gamma)_0=1$ and $(c_\gamma)_r\le(\Gamma(\gamma)r^{1-\gamma})^{-1}$ for $r\ge1$, hence
\begin{equation}
\sum_{r=0}^{p-1}(c_{\gamma}^{(p)})_r
\le 1+\frac{1}{\Gamma(\gamma)}\sum_{r=1}^{p-1}\frac{1}{r^{1-\gamma}}
=1+\frac{1}{\Gamma(\gamma)}\,H_{p-1}^{(1-\gamma)}.
\end{equation}
For $r\ge p$, \Cref{lem:calpha-p-geometric} gives
\begin{equation}
(c_{\gamma}^{(p)})_r\le (c_\gamma)_p(1-\beta_{\gamma,p})^{\,r-p},
\end{equation}
so using $(c_\gamma)_p\le (\Gamma(\gamma)p^{1-\gamma})^{-1}$ from \Cref{lem:bounds-calpha},
\begin{equation}
\sum_{r=p}^{n-1}(c_{\gamma}^{(p)})_r
\le (c_\gamma)_p\sum_{m=0}^{\infty}(1-\beta_{\gamma,p})^{m}
=\frac{(c_\gamma)_p}{1-(1-\beta_{\gamma,p})}
\le \frac{1}{\Gamma(\gamma)\,p^{1-\gamma}\,\beta_{\gamma,p}}.
\end{equation}
Combining the two parts yields the bound in \Cref{eq:l1-sens}.

\noindent\emph{Asymptotic bounds.}
For the asymptotic statement, note that for $\gamma\in(0, 1)$, $\Gamma(\gamma) \geq 1/2$.
Additionally using that $\beta_{\gamma, p} \geq \frac{1-\gamma}{8p}$, we can simplify \Cref{eq:l1-sens} further
\begin{equation}\label{eq:l1-sens-coarse}
    \|C_{\gamma}^{(p)}\|_{1\to 1}
    \leq 1 + 2 H_{p-1}^{(1-\gamma)} + \frac{16 p^{\gamma}}{1-\gamma}
    \leq 3 + \frac{2 p^{\gamma}}{\gamma} + \frac{16 p^{\gamma}}{1-\gamma}.
\end{equation}
where the second step uses \Cref{lem:harm-sum-bound} for $H_{p-1}^{(1-\gamma)} \leq 1 + p^{\gamma}/\gamma$.
The asymptotic statement follows from
\begin{equation}
    3 + \frac{2 p^{\gamma}}{\gamma} + \frac{16 p^{\gamma}}{1-\gamma}
    = \frac{3\gamma(1-\gamma) + (2(1-\gamma) + 16\gamma) p^{\gamma}}{\gamma(1-\gamma)}
    \leq \frac{3 + 18 p^{\gamma}}{\gamma(1-\gamma)}
    = O\left(\frac{p^{\gamma}}{\gamma(1-\gamma)}\right),
\end{equation}
completing the proof.
\end{proof}

\subsection{Error Bounds}

Finally, having derived the requisite matrix norm bounds, we move on to error bounds.
\begin{corollary}\label{cor:rmse}
   Let $\gamma\in[1/2,1)$, $n \geq p \geq 2$ and $1 \leq k \leq \frac{n}{b}$.
   Then 
    \begin{equation*}
        \mathcal{E}(B_{\gamma}^{(p)}, C_{\gamma}^{(p)}) = 
        O\Bigl(
    \sqrt{k}\,p^{\gamma-\frac12}\bigl(\Lambda_\gamma+\tau^{-1/2}\bigr)
    +\sqrt{\tfrac{nk}{p}}\bigl(\sqrt{\tau}+\tau\sqrt{\Lambda_\gamma}\bigr)
    +\sqrt{\tfrac{nk}{b}}
    +\sqrt{\tfrac{k}{b}}\,p^{\gamma}\bigl(\tau^{-1}+\sqrt{\Lambda_\gamma}\bigr)\Bigr).
    \end{equation*}
    where $\tau = 1 - \gamma$ and $\Lambda_\gamma = \min\{\log(ep), 1 / \lvert 2\gamma -1\rvert \} \geq 1$.
\end{corollary}
\begin{proof}
    Statement follows from $\mathcal{E}(B_{\gamma}^{(p)}, C_{\gamma}^{(p)}) = \frac{1}{\sqrt{n}}\|B_\gamma^{(p)}\|_F\cdot\mathrm{sens}_{k,b}(C_{\gamma}^{(p)})$, recalling that
    \begin{equation*}
        \mathrm{sens}_{k,b}(C_{\gamma}^{(p)}) \leq \sqrt{k}\ \| C_{\gamma}^{(p)} \|_{1\to 2} + \sqrt{\frac{k}{b}}\,\|C_{\gamma}^{(p)} \|_{1\to 1}
    \end{equation*}
    from \Cref{lem:sens-toeplitz-bound} (together with \Cref{lem:calpha-p-monotone} for the monotonicity of $C_{\gamma}^{(p)}$'s Toeplitz coefficients), then invoking the bounds \Cref{lem:frob-Bap-bound}, \Cref{lem:l2-sens} and \Cref{lem:l1-sens} for each of the matrix norms involved, then carrying out the algebra.

    To this end, using the requisite matrix norm bounds for $\gamma\in[1/2, 1)$, we have that
    \begin{align*}
        \frac{1}{\sqrt{n}}\| B^{(p)}_\gamma \|_F &= O\left(1 + \tau(\sqrt{\Lambda_\gamma} + \sqrt{n} p^{-\gamma})\right)\\
        \mathrm{sens}_{k,b}(C^{(p)}_\gamma) &= O\left(\!\sqrt{k}p^{\gamma - 1/2}\left(\sqrt{\Lambda_\gamma} + \tau^{-1/2}\right) + \sqrt{\frac{k}{b}}\frac{p^{\gamma}}{\tau}\right).
    \end{align*}
    Multiplying the terms through and using that $\tau \leq 1$, and $\sqrt{\Lambda_\gamma} \leq \Lambda_\gamma$ gives the stated bound.
\end{proof}
While the expression is hard to parse, \Cref{cor:rmse} captures the error of \method in the most important regime of $\gamma$.
For the choice of $\gamma = 1/2$, it recovers the error of BISR~\citep{kalinin2025back}, as expected, and for choices of $\gamma$ closer to $1$, it achieves an error guarantee comparable to DP-$\lambda$CGD.
\Cref{thm:bifr-multi-participation} traces out this trade-off curve more explicitly.

For the remainder of this section, let
\begin{equation*}
    u:=\frac{p}{\sqrt n},\qquad W:=\frac{p\log p}{\sqrt n}=u\log p,\qquad
    \tau^\star:=\frac{\log(1+W)}{\log p},\qquad
    \gamma^\star:=\max\Bigl\{\tfrac12,\ 1-\tau^\star\Bigr\}.
\end{equation*}
We will need the following set of basic inequalities to prove \Cref{thm:bifr-multi-participation}.
\begin{lemma}[Basic inequalities]\label{lem:elem}
Let $u,W>0$, $L\ge\log2$ and $p\ge2$.
\begin{enumerate}[label=(\alph*)]
    \item $W\le(1+W)\log(1+W)$.
    \item $\log(1+uL)/L\le3\log(1+u)$.
    \item $u/(1+uL)\le3\log(1+u)$.
    \item $\log\bigl(1+(p^{1/4}-1)/\log p\bigr)\ge c_1\log p$ for an absolute constant $c_1>0$.
\end{enumerate}
\end{lemma}
\begin{proof}
    \textbf{(a)} The function $(1+W)\log(1+W)-W$ vanishes at $0$ and has derivative $\log(1+W)\ge0$.
    
    For (b) and (c) we use $\log(1+u)\ge u/2$ for $u\le1$.
    \textbf{(b)} If $u\le1$, then:
    \begin{equation}
     \log(1+uL)/L\le u\le2\log(1+u).   
    \end{equation}
    If $u>1$: then $1+uL\le(1+u)(1+L)$, and so
    \begin{equation}
        \log(1+uL)/L\le\log(1+u)/\log2+1\le2\log(1+u)/\log2 \leq 3 \log( 1 + u).
    \end{equation}
    
    \textbf{(c)} If $u\le1$, then:
    \begin{equation}
        u/(1+uL)\le u\le2\log(1+u).
    \end{equation}
    If $u>1$:
    \begin{equation}
        u/(1+uL)\le1/L\le1/\log 2\le\log(1+u)/(\log2)^2 \leq 3\log(1+u).
    \end{equation}
    
    \textbf{(d)} Divided by $\log p$, the left-hand side is continuous and positive on $[2,\infty)$. It tends to
    $\tfrac14$ as $p\to\infty$. Hence its infimum over $[2,\infty)$ is a positive absolute constant.
\end{proof}

\begin{proof}[Proof of \Cref{thm:bifr-multi-participation}]
    Write $L:=\log p$, $\Lambda:=\Lambda_{\gamma^\star}$, $\tau := 1 - \gamma^\star = \min\{ \tau^\star, \tfrac{1}{2}\}$ and
    $Q:=\sqrt{(nk/p)\log(1+u)}$ for the third term of the claimed bound.
    Applying \Cref{cor:rmse} with $\gamma=\gamma^\star$ gives
    \begin{equation}\label{eq:cor-at-gamma-star}
        \mathcal{E}\bigl(B_{\gamma}^{(p)},C_{\gamma}^{(p)}\bigr)
        = O\bigl(T_1+T_2+\sqrt{nk/b}+T_3\bigr),
    \end{equation}
    where
    \begin{equation*}
        T_1:=\sqrt{k}\,p^{\gamma^\star-\frac12}\bigl(\Lambda+\tau^{-1/2}\bigr),\qquad
        T_2:=\sqrt{\tfrac{nk}{p}}\bigl(\sqrt{\tau}+\tau\sqrt{\Lambda}\bigr),\qquad
        T_3:=\sqrt{\tfrac{k}{b}}\,p^{\gamma^\star}\bigl(\tau^{-1}+\sqrt{\Lambda}\bigr).
    \end{equation*}
    Throughout we use $pL=W\sqrt n$ and $p^{1-\tau^\star}=p/(1+W)$.
    It remains to bound $T_1$, $T_2$ and $T_3$, for which we distinguish two cases.
    
    \paragraph{Case 1: $\tau^\star\le\tfrac14$.}
    Then $\tau = \tau^\star$ and $\gamma^\star=1-\tau\geq\tfrac34$, so $\Lambda\le 1/(2\gamma^\star-1)\le2$.
    Since $\tau^{-1/2}\ge2\ge\Lambda$ and $\tau^{-1}\ge4\ge\sqrt\Lambda$, we have that
    \begin{equation}
      T_1\le 2\sqrt k\,p^{\gamma^\star-\frac12}\tau^{-1/2},\qquad T_3\le 2\sqrt{k/b}\,p^{\gamma^\star}/\tau.
    \end{equation}
    Since $\tau\sqrt\Lambda\le\sqrt2\,\tau\le\sqrt2\,\sqrt\tau$, we also have that $T_2\le 3\sqrt{\tau nk/p}$.
    We proceed to show the corresponding bounds.
    \begin{itemize}
        \item Firstly, by \Cref{lem:elem}(a) and then (c),
        \begin{align}
            \frac{p^{2\gamma^\star-1}}{\tau}
            &=\frac{pL}{(1+W)^2\log(1+W)}
            =\frac{\sqrt n\,W}{(1+W)^2\log(1+W)}\\
            &\le\frac{\sqrt n}{1+W}
            =\frac{n}{p}\cdot\frac{u}{1+uL}
            \le 3\,\frac{n}{p}\log(1+u),
        \end{align}
        hence $T_1^2\le 12\,Q^2$.
        \item Secondly, by \Cref{lem:elem}(b), $\tau=\log(1+uL)/L\le3\log(1+u)$, hence $T_2^2\le 27\,Q^2$.
        \item Lastly, by \Cref{lem:elem}(a),
        \begin{equation}
          p^{\gamma^\star}/\tau=p L/\bigl((1+W)\log(1+W)\bigr)=\sqrt n\,W/\bigl((1+W)\log(1+W)\bigr)\le\sqrt n,
        \end{equation}
        hence $T_3\le 2\sqrt{nk/b}$.
    \end{itemize}
    
    \paragraph{Case 2: $\tau^\star > \tfrac14$.}
    The assumption is equivalent to $1+uL>p^{1/4}$, i.e.\ $u>(p^{1/4}-1)/L$, so by \Cref{lem:elem}(d)
    \begin{equation}\label{eq:case2-log}
        \log(1+u)\ge c_1 L .
    \end{equation}
    In particular $\sqrt{nk/p}=O(Q)$ and $\sqrt{nkL/p}=O(Q)$.
    Since $1-\tau^\star<\tfrac34$, we have $\gamma^\star\in[\tfrac12,\tfrac34)$, so $\tau\in(\tfrac14,\tfrac12]$.
    Moreover, we also have that $\Lambda\le\log(ep)=1+L\le3L$.
    Let $z:=p/(1+W)^2$.
    Then
    \begin{align}
        p^{\gamma^\star-\frac12}=\max\bigl\{1,\,p^{\frac12-\tau^\star}\bigr\}=\max\{1,\sqrt z\},\qquad
        \sqrt z&\le\frac{\sqrt p}{W}=\frac{\sqrt{n/p}}{L},
    \end{align}
    and so
    \begin{equation}
        p^{\gamma^\star}=\sqrt p\,p^{\gamma^\star-\frac12}\le\sqrt p+\frac{\sqrt n}{L}.
    \end{equation}
    We are now ready to bound the terms for this case.
    \begin{itemize}
        \item Firstly,
        \begin{align}
            T_1&\le\sqrt k\,(1+\sqrt z)(3L+2)\le 6\sqrt k\,L\,(1+\sqrt z)\\
            &\le 6\sqrt k\,L+6\sqrt{nk/p}=O\bigl(\sqrt k\log p+Q\bigr),
        \end{align}
        using $2\le 3L$.
        \item Secondly, $T_2\le\sqrt{nk/p}\,\bigl(1+\tfrac12\sqrt{3L}\bigr)=O\bigl(\sqrt{nkL/p}\bigr)=O(Q)$.
        \item Lastly, $T_3\le\sqrt{k/b}\,\bigl(\sqrt p+\sqrt n/L\bigr)\bigl(4+\sqrt{3L}\bigr)
        =O\bigl(\sqrt{kp\log p/b}+\sqrt{nk/b}\bigr)$, using that $4\le5\sqrt L$ and $\sqrt L\le 2L$.
    \end{itemize}
    Across both cases (i.e., for all values of $\gamma^\star\in[1/2, 1)$), $T_1, T_2$ and $T_3$ can each be bounded (up to constants) by terms from the bound in \Cref{thm:bifr-multi-participation}.
    This completes the proof.
\end{proof}

\end{document}

%% file: math_commands.tex
\usepackage{amsmath,amsfonts,bm}

\def\eqref#1{equation~\ref{#1}}

\def\1{\bm{1}}

\DeclareMathAlphabet{\mathsfit}{\encodingdefault}{\sfdefault}{m}{sl}
\SetMathAlphabet{\mathsfit}{bold}{\encodingdefault}{\sfdefault}{bx}{n}

